\pdfoutput=1
\documentclass{article}

\usepackage[dvipsnames,table]{xcolor}

\usepackage[activate={true,nocompatibility},final,kerning=true,stretch=5, shrink=29]{microtype}
\usepackage{graphicx}
\usepackage{booktabs}
\usepackage{xfrac}
\usepackage{hyperref}

\usepackage{bm}

\usepackage[scaled=0.92]{helvet}

\usepackage{amsmath}
\usepackage{amssymb}
\usepackage{mathtools}
\usepackage{amsthm}
\usepackage{algorithm}
\usepackage{algpseudocode}

\usepackage[utf8]{inputenc}
\usepackage[T1]{fontenc}
\usepackage[numbers]{natbib}
\usepackage{subcaption}
\usepackage{booktabs}
\usepackage{multirow}
\usepackage{url}
\usepackage{tikz}
\usetikzlibrary{arrows}
\usetikzlibrary{shadows,arrows.meta}
\usepackage{paralist}
\usepackage[stable]{footmisc}
\usepackage[edges]{forest}
\usepackage{fontawesome5}
\usepackage{ifthen}

\usepackage[capitalize,noabbrev]{cleveref}

\usepackage{bm}
\usepackage{amsmath}
\usepackage{amssymb}
\usepackage{amsthm}
\usepackage{amsfonts}
\usepackage{thmtools}		
\usepackage{mleftright}
\usepackage{stmaryrd}
\usepackage{nicefrac}

\let\originalleft\left
\let\originalright\right
\renewcommand{\left}{\mathopen{}\mathclose\bgroup\originalleft}
\renewcommand{\right}{\aftergroup\egroup\originalright}

\theoremstyle{plain}
\newtheorem{theorem}{Theorem}

\newtheorem{proposition}[theorem]{Proposition}

\newtheorem{lemma}[theorem]{Lemma}
\newtheorem{corollary}[theorem]{Corollary}
\theoremstyle{definition}

\usepackage{thm-restate}
\usepackage{siunitx}

\usepackage{pifont}

\usepackage{enumitem}
\setlist[enumerate]{itemsep=0.2ex, topsep=0.5\topsep}
\setlist[description]{itemsep=0.2ex, topsep=0.5\topsep}
\setlist[itemize]{itemsep=0.2ex, topsep=0.5\topsep}

\usepackage{ellipsis}
\usepackage[auth-lg]{authblk}

\makeatletter
\def\thmt@refnamewithcomma #1#2#3,#4,#5\@nil{%
\@xa\def\csname\thmt@envname #1utorefname\endcsname{#3}%
\ifcsname #2refname\endcsname
\csname #2refname\expandafter\endcsname\expandafter{\thmt@envname}{#3}{#4}%
\fi
}
\makeatother
\usepackage{changepage}

\newcommand{\Nb}{\mathbb{N}}

\newcommand{\Rb}{\mathbb{R}}

\newcommand{\kwl}[1]{$#1$\textrm{-}\textsf{WL}\xspace}

\newcommand{\fkwl}[1]{$#1$\textrm{-}\textsf{FWL}}

\newcommand{\kfgt}[1]{ET\xspace}

\newcommand{\wlone}{$1$\textrm{-}\textsf{WL}}

\newcommand{\REL}{\mathsf{recolor}}

\newcommand{\new}[1]{\emph{#1}}

\renewcommand{\vec}[1]{\bm{#1}}
\newcommand{\oms}{\{\!\!\{}
\newcommand{\cms}{\}\!\!\}}

\definecolor{dark2green}{rgb}{0.1, 0.65, 0.3}
\definecolor{dark2orange}{rgb}{0.9, 0.4, 0.}
\definecolor{dark2purple}{rgb}{0.4, 0.4, 0.8}

\definecolor{sns_green}{HTML}{2CA02C}
\definecolor{sns_orange}{HTML}{FF7F0E}
\definecolor{sns_blue}{HTML}{1F77B4}

\colorlet{first_color}{sns_green!40}
\colorlet{second_color}{sns_blue!50}
\colorlet{third_color}{sns_orange!40}

\newcommand{\first}[1]{\fbox{\textbf{\textcolor{sns_green}{#1}}}}
\newcommand{\second}[1]{\underline{\textbf{\textcolor{sns_orange}{#1}}}}
\newcommand{\third}[1]{\textbf{\textcolor{sns_blue}{#1}}}

\newcommand{\triattn}{\mathsf{TriAttention}}

\colorlet{CustomDandelion}{Dandelion!50}
\colorlet{CustomPeriwinkle}{Periwinkle!50}
\colorlet{CustomJungleGreen}{JungleGreen!30}
\colorlet{CustomThistle}{Thistle!50}

\usepackage[mode=buildnew]{standalone}
\usepackage{wrapfig}

\newcommand{\highlight}[2]{#2}

\usepackage{makecell}

\usepackage[final]{neurips_2024}

\usepackage[utf8]{inputenc} %
\usepackage[T1]{fontenc}    %
\usepackage{hyperref}       %
\usepackage{url}            %
\usepackage{booktabs}       %
\usepackage{amsfonts}       %
\usepackage{nicefrac}       %
\usepackage{microtype}      %
\usepackage{xcolor}         %

\title{Towards Principled Graph Transformers}

\author{%
  Luis Müller \\
  RWTH Aachen University \\
  \texttt{luis.mueller@cs.rwth-aachen.de} \\
  \and
  \textbf{Daniel Kusuma} \\
  RWTH Aachen University \\
  \and
  \textbf{Blai Bonet} \\
  Universitat Pompeu Fabra \\
  \and
  \textbf{Christopher Morris} \\
  RWTH Aachen University \\
}

\newcommand{\Omit}[1]{}

\begin{document}

\maketitle

\begin{abstract}
The expressive power of graph learning architectures based on the $k$-dimensional Weisfeiler--Leman (\kwl{k}) hierarchy is well understood. However, such architectures often fail to deliver solid predictive performance on real-world tasks, limiting their practical impact. In contrast, global attention-based models such as graph transformers demonstrate strong performance in practice. However, comparing their expressivity with the \kwl{k} hierarchy remains challenging, particularly since attention-based architectures rely on positional or structural encodings. To address this, we show that the recently proposed Edge Transformer, a global attention model operating on node pairs instead of nodes, has \kwl{3} expressive power when provided with the right tokenization. Empirically, we demonstrate that the Edge Transformer surpasses other theoretically aligned architectures regarding predictive performance and is competitive with state-of-the-art models on algorithmic reasoning and molecular regression tasks while not relying on positional or structural encodings. Our code is available at \url{https://github.com/luis-mueller/towards-principled-gts}. 
\end{abstract}

\section{Introduction}
Graph Neural Networks (GNNs) are the de-facto standard in graph learning \citep{Gil+2017, Sca+2009, Kip+2017, Xu+2018b} but suffer from limited expressivity in distinguishing non-isomorphic graphs in terms of the \new{$1$-dimensional Weisfeiler--Leman algorithm} (\wlone)~\citep{Mor+2019, Xu+2018b}. Hence, recent works introduced \textit{higher-order} GNNs, aligned with the $k$-dimensional Weisfeiler--Leman (\kwl{k}) hierarchy for graph isomorphism testing \citep{Azi+2020, Mar+2019c, Mor+2019, Morris2020b, Mor+2022b}, resulting in more expressivity with an increase in $k > 1$. The \kwl{k} hierarchy draws from a rich history in graph theory and logic \citep{Bab1979, Bab+1979, Bab+1980, Cai+1992,Wei+1968}, offering a deep theoretical understanding of \kwl{k}-aligned GNNs. While theoretically intriguing, higher-order GNNs often fail to deliver state-of-the-art performance on real-world problems, making theoretically grounded models less relevant in practice \citep{Azi+2020, Morris2020b, Mor+2022b}. In contrast, graph transformers~\citep{Glickman+2023,He+2022,ma2023GraphInductiveBiases,rampavsek2022recipe,Ying2021} recently demonstrated state-of-the-art empirical performance. However, they draw their expressive power mostly from positional/structural encodings (PEs), making it difficult to understand these models in terms of an expressivity hierarchy such as the \kwl{k}. While a few works theoretically aligned graph transformers with the \kwl{k} hierarchy \citep{Kim+2021, Kim+2022, Zha+2023}, we are not aware of any works reporting empirical results for \kwl{3}-equivalent graph transformers on established graph learning datasets.

In this work, we aim to set the ground for graph learning architectures that are theoretically aligned with the higher-order Weisfeiler--Leman hierarchy while delivering strong empirical performance and, at the same time, demonstrate that such an alignment creates powerful synergies between transformers and graph learning. Hence, we close the gap between theoretical expressivity and real-world predictive power. To this end, we apply the \new{Edge Transformer} (ET) architecture, initially developed for \new{systematic generalization} problems \citep{Bergen+2021}, to the field of graph learning. Systematic (or compositional) generalization refers to the ability of a model to generalize to complex novel concepts by combining primitive concepts observed during training, posing a challenge to even the most advanced models such as GPT-4 \citep{Dziri+2023+Faith}.

Specifically, we contribute the following:
\begin{enumerate}
    \item We propose a concrete implementation of the Edge Transformer that readily applies to various graph learning tasks.
    \item We show theoretically that this Edge Transformer implementation is as expressive as the \kwl{3} \textit{without} the need for positional/structural encodings.
    \item We demonstrate the benefits of aligning models with the \kwl{k} hierarchy by leveraging well-established results from graph theory and logic to develop a theoretical understanding of systematic generalization in terms of first-order logic statements.
    \item We demonstrate the superior empirical performance of the resulting architecture compared to a variety of other theoretically motivated models, as well as competitive performance compared to state-of-the-art models in molecular regression and neural algorithmic reasoning tasks.
\end{enumerate}

\begin{figure}
    \begin{minipage}{0.475\textwidth}
        \centering
        \vspace{42px}
        \includegraphics[scale=0.85]{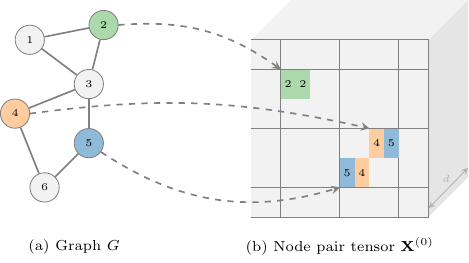}
        \caption{Tokenization of the Edge Transformer. Given a graph $G$, we construct a 3D tensor where we embed information from each node pair into a $d$ dimensional vector.}
        \label{fig:tokenization}
    \end{minipage}
    \hspace{10px}
    \begin{minipage}{0.475\textwidth}
        \centering
        \includegraphics[scale=0.8]{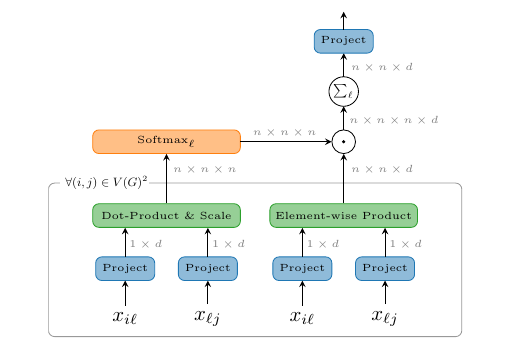}
        \caption{Tensor operations in a single triangular attention head; see \Cref{algo:comparison} for a comparison to standard attention in pseudo-code.}
        \label{fig:triangular_attention_ops}
    \end{minipage}
    
\end{figure}

\section{Related work}
Many graph learning models with higher-order WL expressive power exist, notably $\delta$-$k$-GNNs \citep{Morris2020b}, SpeqNets \citep{Mor+2022b}, $k$-IGNs \citep{Mar+2019b, Mar+2019c}, PPGN \citep{Mar+2019},  and the more recent PPGN++ \citep{Puny+2023}. Moreover, \citet{yaronlipman2020global} devise a low-rank attention module possessing the same power as the folklore \kwl{2} and \citet{Bod+2021b} propose CIN with an expressive power of at least \kwl{3}. For an overview of Weisfeiler--Leman in graph learning, see \citet{Mor+2022}. 

Many graph transformers exist, notably Graphormer \citep{Ying2021} and GraphGPS \citep{rampavsek2022recipe}. However, state-of-the-art graph transformers typically rely on positional/structural encodings, which makes it challenging to derive a theoretical understanding of their expressive power.
The Relational Transformer (RT) \citep{diao+2023+rt} operates over both nodes and edges and, similar to the ET, builds relational representations, that is, representations on edges. Although the RT integrates edge information into self-attention and hence does not need to resort to positional/structural encodings, the RT is theoretically poorly understood, much like other graph transformers. Graph transformers with higher-order expressive power are Graphormer-GD \citep{Zha+2023} and TokenGT \citep{Kim+2022} as well as the higher-order graph transformers in \citet{Kim+2021}.
However, Graphormer-GD is strictly less expressive than the \kwl{3}\citep{Zha+2023}. Further, \citet{Kim+2021} and \citet{Kim+2022} align transformers with $k$-IGNs and, thus, obtain the theoretical expressive power of the corresponding \kwl{k} but do not empirically evaluate their transformers for $k > 2$. In addition, these higher-order transformers suffer from a $\mathcal{O}(n^{2k})$ runtime and memory complexity. For $k=3$, the ET offers provable \kwl{3} expressivity with $\mathcal{O}(n^3)$ runtime and memory complexity, several orders of magnitude more efficient than the corresponding \kwl{3} expressive transformer in \citet{Kim+2022}. 
For an overview of graph transformers, see \citet{Mue+2023}. %

Finally, systematic generalization has recently been investigated both empirically and theoretically \citep{Bergen+2021, Dziri+2023+Faith,Keysers+2020+Measuring,Ren+2023+Improving}. In particular, \citet{Dziri+2023+Faith} demonstrate that compositional generalization is lacking in state-of-the-art transformers such as GPT-4.

\section{Edge Transformers}\label{sec:et}
The ET was originally designed to improve the systematic generalization abilities of machine learning models. To borrow the example from \citet{Bergen+2021}, a model that is presented with relations such as $\textsc{mother}(x, y)$, indicating that $y$ is the mother of $x$, could generalize to a more complex relation $\textsc{grandmother}(x, z)$, indicating that $z$ is the grandmother of $x$ if $\textsc{mother}(x, y) \wedge \textsc{mother}(y, z)$ holds. The particular form of attention used by the ET, which we will formally introduce hereafter, is designed to explicitly model such more complex relations. Indeed, leveraging our theoretical results of \Cref{sec:theory}, in \Cref{sec:logic}, we formally justify the ET for performing systematic generalization. We will now formally define the ET; see \Cref{app:notation} for a complete description of our notation.

In general, the ET operates on a graph $G$ with nodes $V(G)$ and consecutively updates a 3D tensor state $\vec{X} \in \mathbb{R}^{n \times n \times d}$, where $d$ is the embedding dimension and $\vec{X}_{ij}$ or $\vec{X}(\vec{u})$ denotes the representation of the node pair $\vec{u} \coloneqq (i,j) \in V(G)^2$; see \Cref{fig:tokenization} for a visualization of this construction. Concretely, the $t$-th ET layer computes

\begin{equation*}
    \vec{X}^{(t)}_{ij} \coloneqq \mathsf{FFN} \bigl( \vec{X}^{(t-1)}_{ij} + \triattn \bigl(\mathsf{LN}\bigl(\vec{X}^{(t-1)}_{ij} \bigr) \bigr)\bigr),
\end{equation*}

for each node pair $(i,j)$, where $\mathsf{FFN}$ is a feed-forward neural network, $\mathsf{LN}$ denotes layer normalization \citep{Ba+2016} and $\triattn$ is defined as
\begin{equation}\label{eq:2fwl_tf_head}
    \triattn(\vec{X}_{ij}) \coloneqq \sum_{l=1}^n  \alpha_{ilj} \vec{V}_{ilj}, 
\end{equation}
which computes a tensor product between a three-dimensional \textit{attention tensor} $\alpha$ and a three-dimensional \textit{value tensor} $\mathbf{V}$, by multiplying and summing over the second dimension. Here,

\begin{equation}\label{eq:edge_attention_score}
    \alpha_{ilj} \coloneqq \underset{l \in [n]}{\mathsf{softmax}} \Bigl( \frac{1}{\sqrt{d}} \vec{X}_{il}\vec{W}^Q  \bigl(\vec{X}_{lj}\vec{W}^K \bigr)^T \Bigr) \in \mathbb{R}
\end{equation}

is the attention score between the features of tuples $(i, l)$ and $(l, j)$, and

\begin{equation}\label{eq:2fwl_value_fusion}
    \vec{V}_{ilj} \coloneqq \vec{X}_{il}\vec{W}^{V_1} \odot \vec{X}_{lj}\vec{W}^{V_2},
\end{equation}

we call \new{value fusion} of the tuples $(i, l)$ and $(l, j)$ with $\odot$ denoting element-wise multiplication. Moreover,
$\vec{W}^Q, \vec{W}^K, \vec{W}^{V_1}, \vec{W}^{V_2} \in \mathbb{R}^{d \times d}$ are learnable projection matrices; see \Cref{fig:triangular_attention_ops} for an overview of the tensor operations in triangular attention and see \Cref{algo:comparison} for a comparison to standard attention  \citep{Vaswani2017} in pseudo-code. Note that similar to standard attention, triangular attention can be straightforwardly extended to multiple heads. 

\begin{algorithm}
\caption{Comparison between standard attention and triangular attention in \textsc{PyTorch}-like pseudo-code.}\label{algo:comparison}
\begin{minipage}{0.5\textwidth}
\begin{algorithmic}
    \Function{\textsf{attention}}{$\mathbf{X}: n \times d$} \vspace{1px}
        \State $\mathbf{Q}, \mathbf{K}, \mathbf{V} \gets \textsf{linear}(\mathbf{X})\textsf{.chunk}(3)$ \vspace{1px}
        \State \textcolor{sns_orange}{\# no op} \vspace{1px}
        \State $\Tilde{\mathbf{A}} \gets \textsf{einsum}(\textcolor{sns_green}{id, jd \rightarrow ij}, \mathbf{Q}, \mathbf{K})$ \vspace{1px}
        \State $\mathbf{A} \gets \textsf{softmax}(\Tilde{\mathbf{A}} / \sqrt{d}, -1)$ \vspace{2px}
        \State $\mathbf{O} \gets \textsf{einsum}(\textcolor{sns_green}{ij, jd \rightarrow id}, \mathbf{A}, \mathbf{V})$ \vspace{1px}
        \State \Return $\textsf{linear}(\mathbf{O})$ \vspace{1px}
    \EndFunction
\end{algorithmic}
\end{minipage}
\begin{minipage}{0.5\textwidth}
\begin{algorithmic}
    \Function{\textsf{tri\_attention}}{$\mathbf{X}: n \times n \times d$} \vspace{1px}
        \State $\mathbf{Q}, \mathbf{K}, \mathbf{V}^1, \mathbf{V}^2 \gets \textsf{linear}(\mathbf{X})\textsf{.chunk}(4)$ \vspace{1px}
        \State $\mathbf{V} \gets \textsf{einsum}(\textcolor{sns_green}{ild, ljd \rightarrow iljd}, \mathbf{V}^1, \mathbf{V}^2)$ \vspace{1px}
        \State $\Tilde{\mathbf{A}} \gets \textsf{einsum}(\textcolor{sns_green}{ild, ljd \rightarrow ilj}, \mathbf{Q}, \mathbf{K})$ \vspace{1px}
        \State $\mathbf{A} \gets \textsf{softmax}(\Tilde{\mathbf{A}} / \sqrt{d}, -1)$ \vspace{2px}
        \State $\mathbf{O} \gets \textsf{einsum}(\textcolor{sns_green}{ilj, iljd \rightarrow ijd}, \mathbf{A}, \mathbf{V})$ \vspace{1px}
        \State \Return $\textsf{linear}(\mathbf{O})$ \vspace{1px}
    \EndFunction
\end{algorithmic}
\end{minipage}
\end{algorithm}

As we will show in \Cref{sec:theory}, the ET owes its expressive power to the special form of triangular attention. In our implementation of the ET, we use the following tokenization, which is sufficient to obtain our theoretical results.

\paragraph{Tokenization} We consider graphs $G \coloneqq (V(G), E(G), \ell)$ with $n$ nodes and without self-loops, where $V(G)$ is the set of nodes, $E(G)$ is the set of edges, and $\ell \colon V(G) \rightarrow \Nb$ assigns an initial \new{color} to each node. We construct a feature matrix $\vec{F} \in \mathbb{R}^{n \times p}$ that is \new{consistent} with $\ell$, i.e., for nodes $i$ and $j$ in $V(G)$, $\vec{F}_i = \vec{F}_j$ if and only if, $\ell(i) = \ell(j)$. Note that, for a finite subset of $\mathbb{N}$, we can always construct $\vec{F}$, e.g., using a one-hot encoding of the initial colors. Additionally, we consider an edge feature tensor $\vec{E} \in \mathbb{R}^{n \times n \times q}$, where $\vec{E}_{ij}$ denotes the edge feature of the edge $(i, j) \in E(G)$.
If no edge features are available, we randomly initialize learnable vectors $\vec{x}_1, \vec{x}_2 \in \mathbb{R}^q$
and assign $\vec{x}_1$ to $\vec{E}_{ij}$ if $(i,j) \in E(G)$. Further, for all $i \in V(G)$, we assign $\vec{x}_2$ to $\vec{E}_{ii}$.
Lastly, if $(i, j) \not \in E(G)$ and $i \neq j$, we set $\vec{E}_{ij} = \vec{0}$.
We then construct a 3D tensor of input tokens $\vec{X} \in \mathbb{R}^{n \times n \times d}$, such that for node pair $(i,j) \in V(G)^2$,

\begin{equation}\label{eq:higher_order_tokens}
   \vec{X}_{ij} \coloneqq \phi \big ( \begin{bmatrix}
        \vec{E}_{ij} & \vec{F}_i & \vec{F}_j
    \end{bmatrix} \big),
\end{equation}

where $\phi \colon \mathbb{R}^{2p+q} \rightarrow \mathbb{R}^{d}$ is a neural network. Extending \citet{Bergen+2021}, our tokenization additionally considers node features, making it more appropriate for the graph learning setting.

\paragraph{Efficiency} The triangular attention above imposes a $\mathcal{O}(n^3)$  runtime and memory complexity, which is significantly more efficient than other transformers with \kwl{3} expressive power, such as the higher-order transformers in \citet{Kim+2021} and \citet{Kim+2022} with a runtime of $\mathcal{O}(n^6)$. Nonetheless, the ET is still significantly less efficient than most graph transformers, with a runtime of $\mathcal{O}(n^2)$ \citep{ma2023GraphInductiveBiases,rampavsek2022recipe,Ying2021}. Thus, the ET is currently only applicable to mid-sized graphs; see \Cref{sec:limitations} for an extended discussion of this limitation.

\paragraph{Positional/structural encodings} Additionally, GNNs and graph transformers often benefit empirically from added positional/structural encodings \citep{Dwivedi2022, ma2023GraphInductiveBiases, rampavsek2022recipe}. We can easily add PEs to the above tokens with the ET. Specifically, we can encode any PEs for node $i \in V(G)$ as an edge feature in $\vec{E}_{ii}$ and any PEs between a node pair $(i,j) \in V(G)^2$ as an edge feature in $\vec{E}_{ij}$. Note that typically, PEs between pairs of nodes are incorporated during the attention computation of graph transformers \citep{ma2023GraphInductiveBiases, Ying2021}. However, in \Cref{sec:experiments}, we demonstrate that simply adding these PEs to our tokens is also viable for improving the empirical results of the ET.

\paragraph{Readout} Since the Edge Transformer already builds representations on node pairs, making predictions for node pair- or edge-level tasks is straightforward. Specifically, let $L$ denote the number of Edge Transformer layers. Then, for a node pair $(i,j) \in V(G)^2$, we simply readout $\vec{X}^{(L)}_{ij}$, where on the edge-level we restrict ourselves to the case where $(i,j) \in E(G)$. In the case of nodes, we can for example read out the diagonal of $\vec{X}^{(L)}$, that is, the representation for node $i \in V(G)$ is $\vec{X}^{(L)}_{ii}$. In addition to the diagonal readout, we also design a more sophisticated read out strategy for nodes which we describe in \Cref{app:readout}.

With tokenization and readout as defined above, the ET can now be used on many graph learning problems, encoding both node and edge features and making predictions for node pair-, edge-, node-, and graph-level tasks. We refer to a concrete set of parameters of the ET, including tokenization and positional/structural encodings, as a \new{parameterization}. We now move on to our theoretical result, showing that the ET has the same expressive power as the \kwl{3}.

\section{The expressivity of Edge Transformers}\label{sec:theory}
Here, we relate the ET to the \textit{folklore} Weisfeiler--Leman (\fkwl{k}) hierarchy, a variant of the \kwl{k} hierarchy for which, for $k > 2$, \fkwl{(k-1)} is as expressive as \kwl{k} \citep{Gro+2021}. Specifically, we show that the ET can simulate the \fkwl{2}, resulting in \kwl{3} expressive power. To this end, we briefly introduce the \fkwl{2} and then show our result. For detailed background on the \kwl{k} and \fkwl{k} hierarchy, see \Cref{kwl_intro}.

\paragraph{Folklore Weisfeiler--Leman}
Let $G \coloneqq (V(G), E(G), \ell)$ be a node-labeled graph.
The \fkwl{2} colors the tuples from $V(G)^2$, similar to the way the \wlone{} colors nodes \citep{Mor+2019}. In each iteration, $t \geq 0$, the algorithm computes a
\new{coloring} $C^{2,\text{F}}_t \colon V(G)^2 \to \Nb$ and we write $C^{2,\text{F}}_t(i,j)$ or $C^{2,\text{F}}_t(\vec{u})$ to denote the color of tuple $\vec{u} \coloneqq (i,j) \in V(G)^2$ at iteration $t$. For $t=0$, we assign colors to distinguish pairs $(i,j)$ in $V(G)^2$ based on the initial colors $\ell(i), \ell(j)$ of their nodes and whether $(i,j) \in E(G)$ or $i = j$. For a formal definition of the initial node pair colors, see \Cref{app:notation}. Then, for each iteration, $t > 0$, the coloring $C^{2,\text{F}}_{t}$ is defined as

\begin{equation*}
	C^{2,\text{F}}_{t}(i, j) \coloneqq \REL\big( (C^{2,\text{F}}_{t-1}(i,j), \, M_{t-1}(i,j))\big),
\end{equation*}

where $\REL$ injectively maps the above pair to a unique natural number that has not been used in previous iterations and
\begin{equation*}
    M_{t-1}(i,j) \coloneqq \oms (  C^{2, \text{F}}_{t-1}(i, l), \, C^{2, \text{F}}_{t-1}(l, j) ) \mid l \in V(G) \cms.
\end{equation*}

We show that the ET can simulate the \fkwl{2}, resulting in at least \kwl{3} expressive power. Further, we show that the ET is also, at most, as expressive as the \kwl{3}. As a result, we obtain the following theorem; see \Cref{app:proof} for a formal statement and proof details.
\begin{theorem}[Informal]\label{theorem:k_fwl}
    The ET has exactly \kwl{3} expressive power.
\end{theorem}
Note that following previous works \citep{Mar+2019, Morris2020b, Mor+2022b}, our expressivity result is \new{non-uniform} in that our result only holds for an arbitrary but fixed graph size $n$; see \Cref{prop:theorem_forward} and \Cref{prop:theorem_backward} for the complete formal statements and proof of \Cref{theorem:k_fwl}.

In the following, we provide some intuition of how the ET can simulate the \fkwl{2}. Given a tuple $(i,j) \in V(G)^2$, we encode its color at iteration $t$ with $\vec{X}^{(t)}_{ij}$.
Further, to represent the multiset
\begin{equation*}
	\oms (  C^{2, \text{F}}_{t-1}(i, l), \, C^{2, \text{F}}_{t-1}(l, j) ) \mid l \in V(G) \cms,
\end{equation*}

we show that it is possible to encode the pair of colors
\begin{equation*}
   (  C^{2, \text{F}}_{t-1}(i, l), \, C^{2, \text{F}}_{t-1}(l, j) ) \quad \text{via} \quad   \vec{X}^{(t-1)}_{il}\vec{W}^{V_1} \odot \vec{X}^{(t-1)}_{lj}\vec{W}^{V_2},
\end{equation*}
for node $l \in V(G)$. Finally, triangular attention in \Cref{eq:2fwl_tf_head}, performs weighted sum aggregation over the $2$-tuple of colors $(  C^{2, \text{F}}_{t-1}(i, l), \, C^{2, \text{F}}_{t-1}(l, j))$ for each $l$, which we show is sufficient to represent the multiset; see \Cref{app:proof}. For the other direction, namely that the ET has at most \kwl{3} expressive power, we simply show that the $\REL$ function can simulate the value fusion in \Cref{eq:2fwl_value_fusion}, as well as the triangular attention in \Cref{eq:2fwl_tf_head}.

Interestingly, our proofs do not resort to positional/structural encodings. The ET draws its \kwl{3} expressive power from its aggregation scheme, the triangular attention. In~\Cref{sec:experiments}, we demonstrate that this also holds in practice, where the ET performs strongly without additional encodings. In what follows, we use the above results to derive a more principled understanding of the ET in terms of systematic generalization, for which it was originally designed. Thereby, we demonstrate that graph theoretic results can also be leveraged in other areas of machine learning, further highlighting the benefits of theoretically grounded models.

\section{The logic of Edge Transformers}\label{sec:logic}
After borrowing the ET from systematic generalization in the previous section, we now return the favor. Specifically, we use a well-known connection between graph isomorphism and first-order logic to obtain a theoretical justification for systematic generalization reasoning using the ET.  Recalling the example around the \textsc{grandmother} relation composed from the more primitive \textsc{mother} relation in \Cref{sec:et}, \citet{Bergen+2021} go ahead and argue that since self-attention of standard transformers is defined between pairs of nodes, learning explicit representations of $\textsc{grandmother}$ is impossible and that learning such representations implicitly incurs a high burden on the learner. Conversely, the authors argue that since the ET computes triangular attention over triplets of nodes and computes explicit representations between node pairs, the Edge Transformer can systematically generalize to relations such as $\textsc{grandmother}$. While \citet{Bergen+2021} argue the above intuitively, we will now utilize the connection between first-order logic (FO-logic) and graph isomorphism established in \citet{Cai+1992} to develop a theoretical understanding of systematic generalization; see \Cref{app:notation} for an introduction to first-order logic over graphs. We will now briefly introduce the most important concepts in \citet{Cai+1992} and then relate them to systematic generalization of the ET and similar models.

\paragraph{Language and configurations}
Here, we consider FO-logic statements with counting quantifiers and denote with $\mathcal{C}_{k, m}$ the language of all such statements with at most $k$ variables and quantifier depth $m$.
A \new{configuration} is a map between first-order variables and nodes in a graph. Concretely, configurations let us define a statement $\varphi$ in first-order logic, such as three nodes forming a triangle, without speaking about concrete nodes in a graph $G = (V(G), E(G))$. Instead, we can use a configuration to map the three variables in $\varphi$ to nodes $v, w, u \in V(G)$ and evaluate $\varphi$ to determine whether $v, w$ and $u$ form a triangle in $G$. Of particular importance to us are $k$-configurations $f$ where we map $k$ variables $x_1, \dots, x_k$ in a FO-logic statement to a $k$-tuple $\vec{u} \in V(G)^k$ such that $\vec{u} = (f(x_1), \dotsc, f(x_k))$. This lets us now state the following result in \citet{Cai+1992}, relating FO-logic satisfiability to the \fkwl{k} hierarchy. Here, $C_t^{k, \textrm{F}}$ denotes the coloring of the \fkwl{k} after $t$ iterations; see \Cref{app:fkwl_intro} for a precise definition.
\begin{theorem}[Theorem 5.2 \citep{Cai+1992}, informally]
Let $G \coloneqq (V(G), E(G))$ and $H \coloneqq (V(H), E(H))$ be two graphs with $n$ nodes and let $k \geq 1$. Let $f$ be a $k$-configuration mapping to tuple $\vec{u} \in V(G)^k$ and let $g$ be a $k$-configuration mapping to tuple $\vec{v} \in V(H)^k$. Then, for every $t \geq 0$, 
\begin{equation*}
    C_t^{k, \textrm{F}}(\vec{u}) = C_t^{k, \textrm{F}}(\vec{v}),
\end{equation*}
if and only if $\vec{u}$ and $\vec{v}$ satisfy the same sentences in $\mathcal{C}_{k+1,t}$ whose free variables are in $\{x_1,x_2,\ldots,x_k\}$. %
\end{theorem}
Together with \Cref{theorem:k_fwl}, we obtain the above results also for the embeddings of the ET for $k = 2$.
\begin{corollary}
Let $G\coloneqq(V(G), E(G))$ and $H \coloneqq (V(H), E(H))$ be two graphs with $n$ nodes and let $k = 2$. Let $f$ be a $2$-configuration mapping to node pair $\vec{u} \in V(G)^2$ and let $g$ be a $2$-configuration mapping to node pair $\vec{v} \in V(H)^k$. Then, for every $t \geq 0$, 
\begin{equation*}
    \vec{X}^{(t)}(\vec{u}) = \vec{X}^{(t)}(\vec{v}),
\end{equation*}
if and only if $\vec{u}$ and $\vec{v}$ satisfy the same sentences in $\mathcal{C}_{3,t}$ whose free variables are in $\{x_1,x_2\}$. %
\end{corollary}

\paragraph{Systematic generalization}

Returning to the example in \citet{Bergen+2021}, the above result tells us that a model with \fkwl{2} expressive power and at least $t$ layers is equivalently able to evaluate sentences in $\mathcal{C}_{3,t}$, including

\begin{equation*}
 \highlight{sns_blue!15}{\textsc{grandmother}}(x, z) 
 =  \exists y \big( \highlight{sns_orange!15}{\textsc{mother}}(x, y) \wedge \highlight{sns_orange!15}{\textsc{mother}}(y, z)\big),
\end{equation*}

i.e., the grandmother relation, and store this information encoded in some 2-tuple representation $\vec{X}^{(t)}(\vec{u})$, where $\vec{u} = (u, v)$ and $\vec{X}^{(t)}(\vec{u})$ encodes whether $u$ is a grandmother of $v$. %
As a result, we have theoretical justification for the intuitive argument made by \citet{Bergen+2021}, namely that the ET can learn an \textit{explicit} representation of a novel concept, in our example the $\textsc{grandmother}$ relation.

However, when closely examining the language $\mathcal{C}_{3,t}$, we find that the above result allows for an even wider theoretical justification of the systematic generalization ability of the ET. Concretely, we will show that once the ET obtains a representation for a novel concept such as the $\textsc{grandmother}$ relation, at some layer $t$, the ET can re-combine said concept to generalize to even more complex concepts. For example, consider the following relation, which we naively write as

\begin{equation*}
   \highlight{sns_green!15}{\textsc{greatgrandmother}}(x, a) = \exists z \exists y \big(\highlight{sns_orange!15}{\textsc{mother}}(x, y) \wedge \highlight{sns_orange!15}{\textsc{mother}}(y, z) \wedge \highlight{sns_orange!15}{\textsc{mother}}(z, a)\big).
\end{equation*}

At first glance, it seems as though $\textsc{greatgrandmother} \in \mathcal{C}_{4,2}$ but $\textsc{greatgrandmother} \not\in \mathcal{C}_{3,t}$ for any $t \geq 1$. However, notice that the variable $y$ serves merely as an intermediary to establish the $\textsc{grandmother}$ relation. Hence, we can, without loss of generality, write the above as

\begin{equation*}
   \highlight{sns_green!15}{\textsc{greatgrandmother}}(x, a) = \exists y \underbrace{\big(\exists a \big(\highlight{sns_orange!15}{\textsc{mother}}(x, a) \wedge \highlight{sns_orange!15}{\textsc{mother}}(a, y) \big) \big)}_\text{$a$ is re-quantified and temporarily bound} \wedge \highlight{sns_orange!15}{\textsc{mother}}(y, a)\big),
\end{equation*}

i.e., we \new{re-quantify} $a$ to temporarily serve as the mother of $x$ and the daughter of $y$. Afterwards, $a$ is released and again refers to the great grandmother of $x$.
As a result, $\textsc{greatgrandmother} \in \mathcal{C}_{3,2}$ and hence the ET, as well as any other model with at least \fkwl{2} expressive power, is able to generalize to the $\textsc{greatgrandmother}$ relation within two layers, by iteratively re-combining existing concepts, in our example the $\textsc{grandmother}$ and the $\textsc{mother}$ relation. This becomes even more clear, by writing

\begin{equation*}
    \highlight{sns_green!15}{\textsc{greatgrandmother}}(x,a) = \exists y \big(\highlight{sns_blue!15}{\textsc{grandmother}}(x, y) \wedge \highlight{sns_orange!15}{\textsc{mother}}(y, a)\big),
\end{equation*}

where $\textsc{grandmother}$ is a generalized concept obtained from the primitive concept $\textsc{mother}$. To summarize, knowing the expressive power of a model such as the ET in terms of the Weisfeiler--Leman hierarchy allows us to draw direct connections to the logical reasoning abilities of the model. Further, this theoretical connection allows an interpretation of systematic generalization as the ability of a model with the expressive power of at least the \fkwl{k} to iteratively re-combine concepts from first principles (such as the $\textsc{mother}$ relation) as a hierarchy of statements in $\mathcal{C}_{k+1,t}$, containing all FO-logic statements with counting quantifiers, at most $k+1$ variables, and quantifier depth $t$.

\section{Experimental evaluation}\label{sec:experiments}
We now investigate how well the ET performs on various graph-learning tasks. We include tasks on graph-, node-, and edge-level. Specifically, we answer the following questions.
\begin{description}
	\item[Q1] How does the ET fare against other theoretically aligned architectures regarding predictive performance?
	\item[Q2] How does the ET compare to state-of-the-art models? 
	\item[Q3] How effectively can the ET benefit from additional positional/structural encodings?
\end{description}
The source code for our experiments %
is available at \url{https://github.com/luis-mueller/towards-principled-gts}.
To foster research in principled graph transformers such as the ET, we provide accessible implementations of ET, both in PyTorch and Jax.

\paragraph{Datasets}
We evaluate the ET on graph-, node-, and edge-level tasks from various domains to demonstrate its versatility. 

On the graph level, we evaluate the ET on the molecular datasets \textsc{Zinc} (12K), \textsc{Zinc-Full} \citep{Dwi+2020}, \textsc{Alchemy} (12K), and \textsc{PCQM4Mv2} \citep{Hu2021}. Here, nodes represent atoms and edges bonds between atoms, and the task is always to predict one or more molecular properties of a given molecule. Due to their relatively small graphs, the above datasets are ideal for evaluating higher-order and other resource-hungry models.

On the node and edge level, we evaluate the ET on the \textsc{CLRS} benchmark for neural algorithmic reasoning \citep{velickovic+2022+clrs}. Here, the input, output, and intermediate steps of 30 classical algorithms are translated into graph data, where nodes represent the algorithm input and edges are used to encode a partial ordering of the input.
The algorithms of \textsc{CLRS} are typically grouped into eight algorithm classes: Sorting, Searching, Divide and Conquer, Greedy, Dynamic Programming, Graphs, Strings, and Geometry. The task is then to predict the output of an algorithm given its input. This prediction is made based on an encoder-processor-decoder framework introduced by \citet{velickovic+2022+clrs}, which is recursively applied to execute the algorithmic steps iteratively. We will use the ET as the processor in this framework, receiving as input the current algorithmic state in the form of node and edge features and outputting the updated node and edge features, according to the latest version of \textsc{CLRS}, available at \url{https://github.com/google-deepmind/clrs}. As such, the \textsc{CLRS} requires the ET to make both node- and edge-level predictions.

Finally, we conduct empirical expressivity tests on the \textsc{BREC} benchmark \citep{WangZhang2023BREC}. \textsc{BREC} contains 400 pairs of non-isomorphic graphs with up to 198 nodes, ranging from basic, \wlone{} distinguishable graphs to graphs even indistinguishable by \kwl{4}. In addition, \textsc{BREC} comes with its own training and evaluation pipeline. Let $f \colon \mathcal{G} \rightarrow \mathbb{R}^d$ be the model whose expressivity we want to test, where $f$ maps from a set of graphs $\mathcal{G}$ to $\mathbb{R}^d$ for some $d > 0$. Let $(G, H)$ be a pair of non-isomorphic graphs. During training, $f$ is trained to maximize the cosine distance between graph embeddings $f(G)$ and $f(H)$. During the evaluation, \textsc{BREC} decides whether $f$ can distinguish $G$ and $H$ by conducting a Hotelling's T-square test with the null hypothesis that $f$ cannot distinguish $G$ and $H$.

\paragraph{Baselines} On the molecular regression datasets, we compare the ET to an \wlone{} expressive GNN baseline such as GIN(E) \citep{Xu+2019}.

On \textsc{Zinc} (12K), \textsc{Zinc-Full} and \textsc{Alchemy}, we compare the ET to other theoretically-aligned models, most notably higher-order GNNs \citep{Bod+2021b, Morris2020b, Mor+2022b}, Graphormer-GD, with strictly less expressive power than the \kwl{3} \citep{Zha+2023}, and PPGN++, with strictly more expressive power than the \kwl{3}  \citep{Puny+2023} to study \textbf{Q1}. On \textsc{PCQM4Mv2}, we compare the ET to state-of-the-art graph transformers to study \textbf{Q2}.
To study the impact of positional/structural encodings in \textbf{Q3}, we evaluate the ET both with and without relative random walk probabilities (RRWP) positional encodings, recently proposed in \citet{ma2023GraphInductiveBiases}. RRWP encodings only apply to models with explicit representations over node pairs and are well-suited for the ET.

On the \textsc{CLRS} benchmark, we mostly compare to the Relational Transformer (RT) \citep{diao+2023+rt} as a strong graph transformer baseline. Comparing the ET to the RT allows us to study \textbf{Q2} in a different domain than molecular regression and on node- and edge-level tasks. Further, since the RT is similarly motivated as the ET in learning explicit representations of relations, we can study the potential benefits of the ET provable expressive power on the \textsc{CLRS} tasks. In addition, we compare the ET to DeepSet and GNN baselines in \citet{diao+2023+rt} and the single-task Triplet-GMPNN in \citet{ibarz+2022+generalist}.

On the \textsc{BREC} benchmark, we study questions \textbf{Q1} and \textbf{Q2} by comparing the ET to selected models presented in \citet{WangZhang2023BREC}. First, we compare to the $\delta$-$2$-LGNN \citep{Morris2020b}, a higher-order GNN with strictly more expressive power than the \wlone{}. Second, we compare to Graphormer \citep{Ying2021}, an empirically strong graph transformer. Third, we compare to PPGN \citep{Mar+2019} with the same expressive power as the ET. We additionally include the \kwl{3} results on the graphs in \textsc{BREC} to investigate how many \kwl{3} distinguishable graphs the ET can distinguish in \textsc{BREC}.

\paragraph{Experimental setup}
See \Cref{tab:hparams} for an overview of the used hyperparameters.

\begin{table}
\caption{Average test results and standard deviation for the molecular regression datasets. \textsc{Alchemy} (12K) and \textsc{Zinc-Full} over 5 random seeds, \textsc{Zinc} (12K) over 10 random seeds.}
        \label{tab:graph}
        \centering
        \resizebox{0.65\textwidth}{!}{ 	
        \begin{tabular}{lccccc}\toprule
\multirow{2}{*}{\textbf{Model}} & \textsc{Zinc} (12K) & \textsc{Alchemy} (12K) & \textsc{Zinc-Full} \\ \cmidrule{2-4}
 & MAE $\downarrow$ & MAE $\downarrow$ & MAE $\downarrow$ \\\midrule
 GIN(E) \citep{Xu+2018b, Puny+2023} & 0.163 {\tiny ±0.03} & 0.180 {\tiny ±0.006} &  0.180 \tiny ±0.006 \\
\midrule
 CIN \citep{Bod+2021b} & 0.079 \tiny ±0.006 & -- & \second{0.022 \tiny ±0.002} \\
 Graphormer-GD \citep{Zha+2023} & 0.081 \tiny ±0.009 & -- & 0.025 \tiny ±0.004\\
 SignNet \citep{Lim+2022} & 0.084 \tiny ±0.006 & 0.113 \tiny ±0.002& \third{0.024 \tiny ±0.003}\\
 BasisNet \citep{Huang+2023+Stability} & 0.155 \tiny ±0.007 & 0.110 \tiny ±0.001& --\\
 PPGN++ \citep{Puny+2023} & 0.071 \tiny ±0.001 & 0.109 \tiny ±0.001 & \first{0.020 \tiny ±0.001}\\
 SPE \citep{Huang+2023+Stability} & \third{0.069 \tiny ±0.004} & \third{0.108 \tiny ±0.001} & --\\
\midrule
ET & \second{0.062 \tiny ±0.004} & \second{0.099 \tiny ± 0.001} & 0.026 \tiny ±0.003 \\
ET{\tiny+RRWP} & \first{0.059 \tiny ±0.004} & \first{0.098 \tiny ± 0.001} & \third{0.024 \tiny ±0.003} \\
\bottomrule
\end{tabular}}
\end{table}

For \textsc{Zinc} (12K), \textsc{Zinc-Full}, and \textsc{PCQM4Mv2}, we follow the hyperparameters in \citet{ma2023GraphInductiveBiases}.
For \textsc{Alchemy}, we follow standard protocol and split the data according to \citet{Mor+2022b}. Here, we simply adopt the hyper-parameters of \textsc{Zinc} (12K) from \citet{ma2023GraphInductiveBiases} but set the batch size to 64.

We choose the same hyper-parameters as the RT for the \textsc{CLRS} benchmark. Also, following the RT, we train for 10K steps and report results over 20 random seeds.
To stay as close as possible to the experimental setup of our baselines, we integrate our Jax implementation of the ET as a processor into the latest version of the \textsc{CLRS} code base. In addition, we explore the OOD validation technique presented in \citet{JungAhn+2023+TEAM}, where we use larger graphs for the validation set to encourage size generalization. This technique can be used within the \textsc{CLRS} code base through the experiment parameters.

Finally, for \textsc{BREC}, we keep the default hyper-parameters and follow closely the setup used by \citet{WangZhang2023BREC} for PPGN. We found learning on \textsc{BREC} to be quite sensitive to architectural choices, possibly due to the small dataset sizes. As a result, we use a linear layer for the \textsf{FFN} and additionally apply layer normalization onto $\vec{X}_{il}\vec{W}^Q$, $\vec{X}_{lj}\vec{W}^K$ in \Cref{eq:edge_attention_score} and $\vec{V}_{ilj}$ in \Cref{eq:2fwl_value_fusion}. 

For \textsc{Zinc} (12K), \textsc{Zinc-Full}, \textsc{PCQM4Mv2}, \textsc{CLRS}, and \textsc{BREC}, we follow the standard train/validation/test splits. For \textsc{Alchemy}, we split the data according to the splits in \citet{Mor+2022b}, the same as our baselines.

All experiments were performed on a mix of A10, L40, and A100 NVIDIA GPUs. For each run, we used at most 8 CPU cores and 64\,GB of RAM, with the exception of \textsc{PCQM4Mv2} and \textsc{Zinc-Full}, which were trained on 4 L40 GPUs with 16 CPU cores and 256 GB RAM.

\begin{table}
    \caption{Average test micro F1 of different algorithm classes and average test score of all algorithms in CLRS over ten random seeds; see \Cref{sec:clrs_scores} for test scores per algorithm and \Cref{sec:clrs_stddev} for details on the standard deviation.
    }
    \label{tab:clrs}
            \centering
\resizebox{\textwidth}{!}{
\begin{tabular}{lcccccccc}\toprule \textbf{Algorithm} & Deep Sets \citep{diao+2023+rt} & GAT \citep{diao+2023+rt} & MPNN \citep{diao+2023+rt} & PGN \citep{diao+2023+rt} & RT \citep{diao+2023+rt} & \makecell{Triplet-\\GMPNN \citep{ibarz+2022+generalist}} & ET (ours) \\ \midrule
Sorting & \second{68.89} & 21.25 & 27.12 & 28.93 & 50.01 & \third{60.37} & \first{82.26} \\
Searching & 50.99 & 38.04 & 43.94 & \third{60.39} & \first{65.31} & 58.61 & \second{63.00} \\ 
DC & 12.29 & 15.19 & 16.14 & 51.30 & \third{66.52} & \first{76.36} & \third{64.44} \\ 
Greedy & 77.83 & 75.75 & \second{89.40} & 76.72 & \third{85.32} & \first{91.21} & 81.67 \\ 
DP & 68.29 & 63.88 & 68.81 & 71.13 & \second{83.20} & \third{81.99} & \first{83.49} \\ 
Graphs & 42.09 & 55.53 & 63.30 & 64.59 & \third{65.33} & \second{81.41} & \first{86.08} \\ 
Strings & 2.92 & 1.57 & 2.09 & 1.82 & \third{32.52} & \second{49.09} & \first{54.84} \\ 
Geometry & 65.47 & 68.94 & 83.03 & 67.78 & \third{84.55} & \first{94.09} & \second{88.22} \\ \midrule
Avg. class & 48.60 & 41.82 & 49.23 & 52.83 & \third{66.60} & \second{74.14} & \first{75.51} \\ 
All algorithms & 50.29 & 48.08 & 55.15 & 56.57 & \third{66.18} & \second{75.98} & \first{80.13} \\ 
\bottomrule
\end{tabular}}
\end{table}

\begin{table}[ht]
\caption{Number of distinguished pairs of non-isomorphic graphs on the \textsc{BREC} benchmark over 10 random seeds with standard deviation. Baseline results (over 1 random seed) are taken from \citet{WangZhang2023BREC}. For reference, we also report the number of graphs distinguishable by \kwl{3}.}
        \label{tab:brec}
        \centering
\resizebox{0.65\textwidth}{!}{ 	
        \begin{tabular}{lccccc}\toprule
\textbf{Model} & Basic & Regular & Extension & CFI & \textit{All} \\ \midrule
$\delta$-$2$-LGNN & 60 & 50 & 100 & 6 & \third{216} \\
PPGN & 60 & 50 & 100 & 23 & \second{233} \\
Graphormer & 16 & 12 & 41 & 10 & 79 \\ \midrule
ET & 60 \tiny ± 0.0 & 50  \tiny ±0.0 & 100  \tiny ±0.0 & 48.1 \tiny ±1.9 & \first{258.1 \tiny ±1.9} \\\midrule \midrule
\kwl{3} & 60 & 50 & 100 & 60 & 270 \\
\bottomrule
\end{tabular}}
\end{table}

\begin{table}
\begin{minipage}{0.5\textwidth}
\caption{Average validation MAE on the \textsc{PCQM4Mv2} benchmark over a single random seed.}
\label{tab:pcqm}
\centering
\resizebox{\textwidth}{!}{ 	
\begin{tabular}{lcc}\toprule
    \textbf{Model} & Val. MAE ($\downarrow$) & \# Params \\ \midrule
    EGT \citep{Hussain+2022} & 0.0869  & 89.3M \\
    GraphGPS{\tiny Small} \citep{rampavsek2022recipe} & 0.0938 & 6.2M \\
    GraphGPS{\tiny Medium} \citep{rampavsek2022recipe} & 0.0858 & 19.4M \\
    TokenGT{\tiny ORF} \citep{Kim+2022} & 0.0962 & 48.6M \\
    TokenGT{\tiny Lap} \citep{Kim+2022} & 0.0910 & 48.5M \\
    Graphormer \citep{Ying2021} & 0.0864 & 48.3M \\
    GRIT \citep{ma2023GraphInductiveBiases} & 0.0859 & 16.6M \\
    GPTrans-L & \first{0.0809} & 86.0M \\ \midrule
    ET & \third{0.0840} & 16.8M \\
    ET{\tiny+RRWP} & \second{0.0832} & 16.8M \\
    \bottomrule
\end{tabular}}
\end{minipage}
\begin{minipage}{0.5\textwidth}
\caption{\textsc{Zinc} (12K) leaderboard.}
\label{tab:leaderboard}
\centering
\resizebox{0.725\textwidth}{!}{
            \begin{tabular}{lcc}\toprule
\multirow{2}{*}{\textbf{Model}} & \textsc{Zinc} (12K)\\ \cmidrule{2-2}
 & MAE $\downarrow$\\\midrule
SignNet \citep{Lim+2022} & 0.084 \tiny ±0.006\\
SUN \citep{frasca2022understanding} & 0.083 \tiny ±0.003 \\
Graphormer-GD \citep{Zha+2023} & 0.081 \tiny ±0.009\\
CIN \citep{Bod+2021b} & 0.079 \tiny ±0.006\\
Graph-MLP-Mixer \citep{He+2022} & 0.073 \tiny ±0.001 \\
PPGN++ \citep{Puny+2023} & 0.071 \tiny ±0.001\\
GraphGPS \citep{rampavsek2022recipe} & 0.070 \tiny ±0.004 \\
SPE \citep{Huang+2023+Stability} & 0.069 \tiny ±0.004 \\
Graph Diffuser \citep{Glickman+2023} & 0.068 \tiny ±0.002 \\
Specformer \citep{Bo+2023} & 0.066 \tiny ±0.003 \\
GRIT \citep{ma2023GraphInductiveBiases} & \first{0.059 \tiny ±0.002} \\
\midrule
ET & \third{0.062 \tiny ±0.004}\\
ET{\tiny+RRWP} & \second{0.059 \tiny ±0.004}\\
\bottomrule
\end{tabular} 
}
\end{minipage}
\end{table}

\paragraph{Results and discussion}
In the following, we answer questions \textbf{Q1} to \textbf{Q3}. We highlight \first{first}, \second{second}, and \third{third} best results in each table. 

We compare results on the molecular regression datasets in \Cref{tab:graph}. On \textsc{Zinc} (12K) and \textsc{Alchemy}, the ET outperforms all baselines, even without using positional/structural encodings, positively answering \textbf{Q1}. Interestingly, on \textsc{Zinc-Full}, the ET, while still among the best models, does not show superior performance.
Further, the RRWP encodings we employ on the graph-level datasets improve the performance of the ET on all three datasets, positively answering \textbf{Q3}. Moreover, in \Cref{tab:leaderboard}, we compare the ET with a variety of graph learning models on \textsc{Zinc} (12K), demonstrating that the ET is highly competitive with state-of-the-art models. We observe similarly positive results in \Cref{tab:pcqm} where the ET outperforms strong graph transformer baselines such as GRIT \citep{ma2023GraphInductiveBiases}, GraphGPS \citep{rampavsek2022recipe} and Graphormer \citep{Ying2021} on \textsc{PCQM4Mv2}. As a result, we can positively answer \textbf{Q2}.

In \Cref{tab:clrs}, we compare results on \textsc{CLRS} where the ET performs best when averaging all tasks or when averaging all algorithm classes, improving over RT and Triplet-GMPNN.
Additionally, the ET performs best on 4 algorithm classes and is among the top 3 in 7/8 algorithm classes. Interestingly, only some models are best on a majority of algorithm classes.
These results indicate a benefit of the ETs' expressive power on this benchmark,
adding to the answer of \textbf{Q2}. Further, see \Cref{tab:clrs_ood} in \Cref{app:clrs_ood} for additional results using the OOD validation technique.

Finally, on the \textsc{BREC} benchmark, we observe that the ET cannot distinguish all graphs distinguishable by \kwl{3}. At the same time, the ET distinguishes more graphs than PPGN, the other \kwl{3} expressive model, providing an additional positive answer to \textbf{Q1}; see \Cref{tab:brec}.
Moreover, the ET distinguishes more graphs than $\delta$-2-LGNN and outperforms Graphormer by a large margin, again positively answering \textbf{Q2}. Overall, the positive results of the ET on \textsc{BREC} indicate that the ET is well able to leverage its expressive power empirically.

\section{Limitations}\label{sec:limitations}
While proving to be a strong and versatile graph model, the ET has an asymptotic runtime and memory complexity of $\mathcal{O}(n^3)$, which is more expensive than most state-of-the-art models with linear or quadratic runtime and memory complexity.
We emphasize that due to the runtime and memory complexity of the \kwl{k}, a trade-off between expressivity and efficiency is likely unavoidable. 
At the same time, the ET is highly parallelizable and runs efficiently on modern GPUs. We hope that innovations for parallelizable neural networks can compensate for the asymptotic runtime and memory complexity of the ET. In \Cref{fig:et_timing} in the appendix, we find that we can use low-level GPU optimizations, available for parallelizable neural networks out-of-the-box, to dampen the cubic runtime and memory scaling of the ET; see \Cref{app:runtime_and_memory} for runtime and memory experiments and an extended discussion.

\section{Conclusion}
We established a previously unknown connection between the Edge Transformer and \kwl{3}, and enabled the Edge Transformer for various graph learning tasks, including graph-, node-,~and edge-level tasks. We also utilized a well-known connection between graph isomorphism testing and first-order logic to derive a theoretical interpretation of systematic generalization.
We demonstrated empirically that the Edge Transformer is a promising architecture for graph learning, outperforming other theoretically aligned architectures and being among the best models on \textsc{Zinc} (12K), \textsc{PCQM4Mv2} and \textsc{CLRS}. Furthermore, the ET is a graph transformer that does not rely on positional/structural encodings for strong empirical performance. Future work could further explore the potential of the Edge Transformer in neural algorithmic reasoning and molecular learning by improving its scalability to larger graphs, in particular through architecture-specific low-level GPU optimizations and model parallelism.

\begin{ack}
CM and LM are partially funded by a DFG Emmy Noether grant (468502433) and RWTH Junior Principal Investigator Fellowship under Germany’s Excellence Strategy. We thank Erik Müller for crafting the figures.
\end{ack}

\bibliography{reference}

\newpage

\appendix

\section{Implementation details}
Here, we present details about implementing the ET in practice.

\subsection{Node-level readout}\label{app:readout}
In what follows, we propose a pooling method from node pairs to nodes, which allows us also to make predictions for node- and graph-level tasks. For each node $i \in V(G)$, we compute

\begin{equation*}
   \mathsf{ReadOut}(i) \coloneqq \sum_{j \in [n]} \rho_1 \Bigl( \vec{X}^{(L)}_{ij} \Bigr) + \rho_2 \Bigl( \vec{X}^{(L)}_{ji} \Bigr), 
\end{equation*}

where $\rho_1, \rho_2$ are neural networks and $\vec{X}^{(L)}$ is the node pair tensor after $L$ ET layers. We apply $\rho_1$ to node pairs where node $i$ is at the first position and $\rho_2$ to node pairs where node $i$ is at the second position. We found that making such a distinction has positive impacts on empirical performance. Then, for graph-level predictions, we first compute node-level readout as above and then use common graph-level pooling functions such as \texttt{sum} and \texttt{mean} \citep{Xu+2018b} or \texttt{set2seq} \citep{Vin+2016} on the resulting node representations. We use this readout method in our molecular regression experiments in \Cref{sec:experiments}.

\section{Experimental details}\label{app:experiments}
\begin{table*}
    \centering
     \caption{Hyperparameters of the Edge Transformer across all datasets.}
    \resizebox{\textwidth}{!}{
    \begin{tabular}{lcccccccc}\toprule
\textbf{Hyperparameter} & \textsc{Zinc}(12K) & \textsc{Alchemy} & \textsc{Zinc-Full} & \textsc{CLRS} & \textsc{BREC} & \textsc{PCQM4Mv2}\\
\midrule
Learning rate & 0.001 & 0.001 & 0.001 & 0.00025 & 0.0001 & 0.0002 \\
Grad. clip norm & 1.0 & 1.0 & 1.0 & 1.0 & -- & 5.0 \\
Batch size & 32 & 64 & 256 & 4 & 16 & 256\\
Optimizer & AdamW & Adam & AdamW & Adam & Adam & AdamW\\
\midrule
Num. layers & 10 & 10 & 10 & 3 & 5 & 10 \\
Hidden dim. & 64 & 64 & 64 & 192 & 32 & 384 \\
Num. heads & 8 & 8 & 8 & 12 & 4 & 16 \\
Activation & $\textsc{GELU}$ & $\textsc{GELU}$ & $\textsc{GELU}$ & \textsc{ReLU} & -- & \textsc{GELU} \\
Pooling & \textsc{sum} & \textsc{sum} & \textsc{sum} & -- & -- & \textsc{sum} \\
RRWP dim. & 32 & 32 & 32 & -- & -- & 128 \\
\midrule
Weight decay & 1e-5 & 1e-5 & 1e-5 & -- & 0.0001 & 0.1 \\
Dropout & 0.0 & 0.0 & 0.0 & 0.0 & 0.0 & 0.1 \\
Attention dropout & 0.2 & 0.2 & 0.2 & 0.0 & 0.0 & 0.1 \\\midrule
\# Steps & -- & -- & -- & 10K & -- & 2M \\
\# Warm-up steps & -- & -- & -- & 0 & -- & 60K \\
\# Epochs & 2K & 2K & 1K & -- & 20 & -- \\
\# Warm-up epochs & 50 & 50 & 50 & -- & 0 & -- \\
\# RRWP steps & 21 & 21 & 21 & -- & -- & 22 \\
\bottomrule
\end{tabular}
}
    \label{tab:hparams}
\end{table*}
\Cref{tab:hparams} gives an overview of selected hyper-parameters for all experiments.

See \Cref{app:clrs_ood} through \Cref{sec:clrs_stddev} for detailed results on the CLRS benchmark.
Note that in the case of CLRS, we evaluate in the single-task setting where we train a new set of parameters for each concrete algorithm, initially proposed in \textsc{CLRS}, to be able to compare against graph transformers fairly. We leave the multi-task learning proposed in \citet{ibarz+2022+generalist} for future work.

\subsection{Data source and license}\label{app:source_license}
\textsc{Zinc} (12K), \textsc{Alchemy} (12K) and \textsc{Zinc-Full} are available at \url{https://pyg.org} under an MIT license.
\textsc{PCQM4Mv2} is available at \url{https://ogb.stanford.edu/docs/lsc/pcqm4mv2/} under a CC BY 4.0 license.
The \textsc{CLRS} benchmark is available at \url{https://github.com/google-deepmind/clrs} under an Apache 2.0 license.
The \textsc{BREC} benchmark is available at \url{https://github.com/GraphPKU/BREC} under an MIT license.

\subsection{Experimental results OOD validation in CLRS}\label{app:clrs_ood}
In \Cref{tab:clrs_ood}, following \citep{JungAhn+2023+TEAM}, we present additional experimental results on \textsc{CLRS} when using graphs of size 32 in the validation set. 
We compare to both the Triplet-GMPNN \citep{ibarz+2022+generalist}, as well as the TEAM \citep{JungAhn+2023+TEAM} baselines.
In addition, in \Cref{fig:clrs_ood_improve}, we present a comparison of the improvements resulting from the OOD validation technique, comparing Triplet-GMPNN and the ET. Finally, in \Cref{tab:clrs_proc_agnostic}, we compare different modifications to the CLRS training setup that are agnostic to the choice of processor.

\begin{table}
    \caption{Average test scores for the different algorithm classes and average test scores of all algorithms in CLRS \textbf{with the OOD validation technique} over 10 seeds; see \Cref{sec:clrs_scores} for test scores per algorithm and \Cref{sec:clrs_stddev} for details on the standard deviation. Baseline results for Triplet-GMPNN and TEAM are taken from \citet{JungAhn+2023+TEAM}. Results in \%.}
    \label{tab:clrs_ood}
            \centering
\resizebox{0.7\textwidth}{!}{
\begin{tabular}{lcccccccccc}\toprule \textbf{Algorithm} & Triplet-GMPNN & TEAM & ET (ours) \\ \midrule
Sorting & 72.08 & 68.75 & \textbf{88.35} \\
Searching & 61.89 & 63.00 & \textbf{80.00}\\ 
DC & 65.70 & 69.79 & \textbf{74.70}\\ 
Greedy & 91.21 & \textbf{91.80} & 88.29\\ 
DP & \textbf{90.08} & 83.61 & 84.69\\ 
Graphs & 77.89 & 81.86 & \textbf{89.89}\\ 
Strings & 75.33 &  \textbf{81.25} & 51.22\\ 
Geometry & 88.02 &  \textbf{94.03 }& 89.68\\ \midrule
Avg. algorithm class & 77.48 & 79.23 &\textbf{ 80.91} \\ 
All algorithms & 78.00 & 79.82 & \textbf{85.01} \\ 
\bottomrule
\end{tabular}}
\end{table}

\begin{figure}
    \centering
    \includegraphics[width=0.75\linewidth]{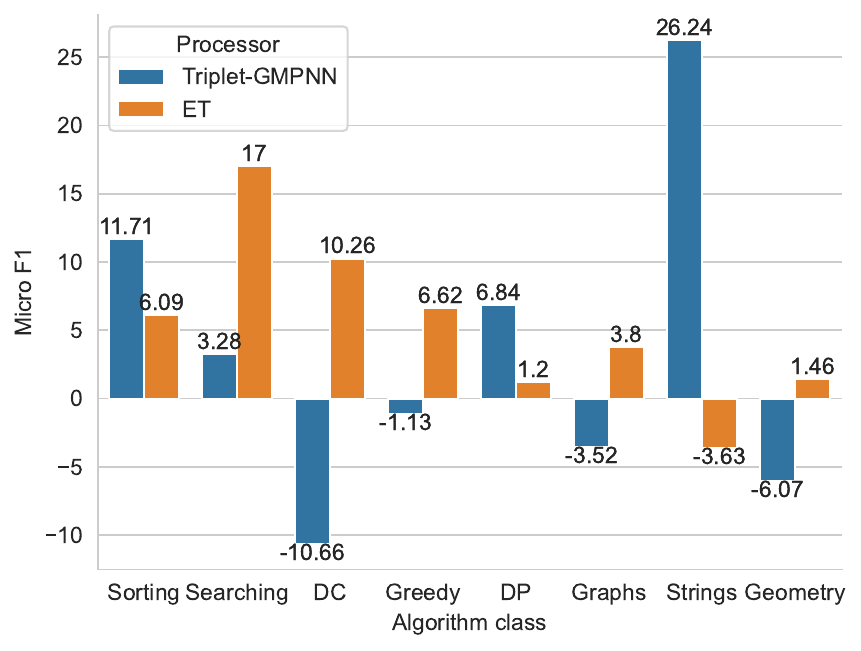}
    \caption{Difference in micro F1 with and without the OOD validation technique in \citet{JungAhn+2023+TEAM}, for Triplet-GMPNN \citep{ibarz+2022+generalist} and ET, respectively.}
    \label{fig:clrs_ood_improve}
\end{figure}

\begin{table}
    \caption{CLRS-30 Processor-agnostic modifications.}
    \label{tab:clrs_proc_agnostic}
            \centering
\resizebox{\textwidth}{!}{
\begin{tabular}{lcccccccccc}\toprule \textbf{Processor} & Markov \citep{bohde+2024+gforgetnet} & OOD Validation \citep{JungAhn+2023+TEAM} & \textit{Avg. algorithm class} & \textit{All algorithms}\\ \midrule
Triplet-GMPNN & \ding{51} & \ding{55}  & \second{79.75}  & \second{82.89} \\
Triplet-GMPNN & \ding{55} & \ding{51}  & 77.65  & 78.00 \\
TEAM & \ding{55} & \ding{51}  & \third{79.23}  & \third{79.82} \\
ET & \ding{55} & \ding{51}  & \first{80.91}  & \first{85.02} \\ 
\bottomrule 
\end{tabular}}
\end{table}

\subsection{CLRS test scores}\label{sec:clrs_scores}
We present detailed results for the algorithms in \textsc{CLRS}. See \Cref{tab:clrs_detail_dnc} for divide and conquer algorithms, \Cref{tab:clrs_detail_dp} for dynamic programming algorithms, \Cref{tab:clrs_detail_geometry} for geometry algorithms, \Cref{tab:clrs_detail_greedy} for greedy algorithms, \Cref{tab:clrs_detail_search} for search algorithms, \Cref{tab:clrs_detail_sorting} for sorting algorithms, and \Cref{tab:clrs_detail_string} for string algorithms.

\begin{table}
\centering
\caption{Detailed test scores for the ET on sorting algorithms.}
\resizebox{\textwidth}{!}{\begin{tabular}{lcccc}
\toprule
\textbf{Algorithm} & F1-score(\%) & Std. dev.(\%) & F1-score(\%)(OOD) & Std. dev.(\%) (OOD)\\
\midrule
Bubble Sort & 93.60 & 3.87 & 87.44 & 13.48 \\
Heapsort & 64.36 & 22.41 & 80.96 & 12.97 \\
Insertion Sort & 85.71 & 20.68 & 91.74 & 6.83 \\
Quicksort & 85.37 & 8.70 & 93.25 & 9.10 \\
\midrule
\textit{Average} &  82.26 & 13.92 & 88.35 & 10.58 \\
\bottomrule
\end{tabular}}
\label{tab:clrs_detail_sorting}
\end{table}

\begin{table}
\centering \caption{Detailed test scores for the ET on search algorithms.}
\resizebox{\textwidth}{!}{\begin{tabular}{lcccc}
\toprule
\textbf{Algorithm} & F1-score(\%) & Std. dev.(\%) & F1-score(\%)(OOD) & Std. dev.(\%) (OOD)\\
\midrule
Binary Search & 79.96 & 11.66 & 90.84 & 2.71 \\
Minimum & 96.88 & 1.74 & 97.94 & 0.87 \\
Quickselect & 12.43 & 11.72 & 52.64 & 22.04 \\
\midrule
\textit{Average} & 63.00 & 8.00 & 80.00 & 8.54 \\
\bottomrule
\end{tabular}}
\label{tab:clrs_detail_search}
\end{table}

\begin{table}
\centering
\caption{Detailed test scores for the ET on divide and conquer algorithms.}
\resizebox{\textwidth}{!}{\begin{tabular}{lcccc} \toprule \textbf{Algorithm} & F1-score(\%) & Std. dev.(\%) & F1-score(\%)(OOD) & Std. dev.(\%) (OOD)\\
\midrule
\small Find Max. Subarray Kadande & 64.44 & 2.24 & 74.70 & 2.59 \\
\midrule
\textit{Average} & 64.44 & 2.24 & 74.70 & 2.59 \\
\bottomrule
\end{tabular}}
\label{tab:clrs_detail_dnc}
\end{table}

\begin{table}
\centering
\caption{Detailed test scores for the ET on dynamic programming algorithms.}
\resizebox{\textwidth}{!}{\begin{tabular}{lcccc}
\toprule
\textbf{Algorithm} & F1-score(\%) & Std. dev.(\%) & F1-score(\%)(OOD) & Std. dev.(\%) (OOD)\\
\midrule
LCS Length & 88.67 & 2.05 & 88.97 & 2.06 \\
Matrix Chain Order & 90.11 & 3.28 & 90.84 & 2.94 \\
Optimal BST & 71.70 & 5.46 & 74.26 & 10.84 \\
\midrule
\textit{Average} & 83.49 & 3.60 & 84.68 & 5.28 \\
\bottomrule
\end{tabular}}
\label{tab:clrs_detail_dp}
\end{table}

\begin{table}
\centering
\caption{Detailed test scores for the ET on geometry algorithms.}
\resizebox{\textwidth}{!}{\begin{tabular}{lcccc}
\toprule
\textbf{Algorithm} & F1-score(\%) & Std. dev.(\%) & F1-score(\%)(OOD) & Std. dev.(\%) (OOD)\\
\midrule
Graham Scan & 92.23 & 2.26 & 96.09 & 0.96 \\
Jarvis March & 89.09 & 8.92 & 95.18 & 1.46 \\
Segments Intersect & 83.35 & 7.01 & 77.78 & 1.16 \\
\midrule
\textit{Average} & 88.22 & 6.09 & 89.68 & 1.19  \\
\bottomrule
\end{tabular}}
\label{tab:clrs_detail_geometry}
\end{table}

\begin{table}
\centering
\caption{Detailed test scores for the ET on graph algorithms.}
\resizebox{\textwidth}{!}{\begin{tabular}{lcccc}
\toprule
\textbf{Algorithm} & F1-score(\%) & Std. dev.(\%) & F1-score(\%)(OOD) & Std. dev.(\%) (OOD)\\
\midrule
Articulation Points & 93.06 & 0.62 & 95.47 & 2.35 \\
Bellman-Ford & 89.96 & 3.77 & 95.55 & 1.65 \\
BFS & 99.77 & 0.30 & 99.95 & 0.08 \\
Bridges & 91.95 & 10.00 & 98.28 & 2.64 \\
DAG Shortest Paths & 97.63 & 0.85 & 98,43 & 0.65 \\
DFS & 65.60 & 17.98 & 57.76 & 14.54 \\
Dijkstra & 91.90 & 2.99 & 97.32 & 7.32 \\
Floyd-Warshall & 61.53 & 5.34 & 83.57 & 1.79 \\
MST-Kruskal & 84.06 & 2.14 & 87.21 & 1.45 \\
MST-Prim & 93.02 & 2.41 & 93.00 & 1.61 \\
SCC & 65.80 & 8.13 & 74.58 & 5.31 \\
Topological Sort & 98.74 & 2.24 & 97.53 & 2.31 \\
\midrule
\textit{Average} & 86.08 & 4.73 & 89.92 & 3.02 \\
\bottomrule
\end{tabular}}
\label{tab:clrs_detail_graph}
\end{table}

\begin{table}
\centering
\caption{Detailed test scores for the ET on greedy algorithms.}
\resizebox{\textwidth}{!}{\begin{tabular}{lcccc} \toprule \textbf{Algorithm} & F1-score(\%) & Std. dev.(\%) & F1-score(\%)(OOD) & Std. dev.(\%) (OOD)\\
\midrule
Activity Selector & 80.12 & 12.34 & 91.72 & 2.35 \\
Task Scheduling & 83.21 & 0.30 & 84.85 & 2.83 \\
\midrule
\textit{Average} & 81.67 & 6.34 & 88.28 & 2.59 \\
\bottomrule
\end{tabular}}
\label{tab:clrs_detail_greedy}
\end{table}

\begin{table}
\centering
\caption{Detailed test scores for the ET on string algorithms.}
\resizebox{\textwidth}{!}{\begin{tabular}{lcccc}
\toprule
\textbf{Algorithm} & F1-score(\%) & Std. dev.(\%) & F1-score(\%)(OOD) & Std. dev.(\%) (OOD)\\
\midrule
KMP Matcher & 10.47 & 10.28 & 8.67 & 8.14 \\
Naive String Match & 99.21 & 1.10 & 93.76 & 6.28\\
\midrule
\textit{Average} & 54.84 & 5.69 & 51.21 & 7.21 \\
\bottomrule
\end{tabular}}
\label{tab:clrs_detail_string}
\end{table}

\subsection{CLRS test standard deviation}\label{sec:clrs_stddev}
We compare the standard deviation of Deep Sets, GAT, MPNN, PGN, RT, and ET following the comparison in \citet{diao+2023+rt}.
\Cref{tab:clrs_stddev_over_models} compares the standard deviation over all algorithms in the CLRS benchmark. We observe that the ET has the lowest overall standard deviation.
The table does not contain results for Triplet-GMPNN \citep{ibarz+2022+generalist} since we do not have access to the test results for each algorithm on each seed that are necessary to compute the overall standard deviation.
However, \Cref{tab:clrs_stddev_per_class} compares the standard deviation per algorithm class between Triplet-GMPNN and the ET. We observe that Triplet-GMPNN and the ET have comparable standard deviations except for search and string algorithms, where Triplet-GMPNN has a much higher standard deviation than the ET.

\begin{table}
    \centering
    \caption{Standard deviation of Deep Sets, GAT, MPNN, PGN, RT, and ET (over all algorithms and all seeds).}
    \begin{tabular}{lc}
    \toprule
       \textbf{Model}  & Std. Dev. (\%) \\\midrule
      Deep Sets & 29.3 \\
      GAT & 32.3 \\
      MPNN & 34.6 \\
      PGN & 33.1 \\
      RT & 29.6 \\
      \midrule
      ET & \textbf{26.6} \\
    \bottomrule
    \end{tabular}
    \label{tab:clrs_stddev_over_models}
\end{table}

\begin{table}
    \centering
    \caption{Standard deviation per algorithm class of Triplet-GMPNN (over 10 random seeds) as reported in \citet{ibarz+2022+generalist} and ET (over 10 random seeds). Results in \%.}
    \begin{tabular}{lcc}
    \toprule
       \textbf{Algorithm class}  & Triplet-GMPNN & ET  \\\midrule
      Sorting   & 12.16 & 15.57 \\
      Searching & 24.34 & 3.51 \\
      Divide and Conquer & 1.34& 2.46 \\
      Greedy & 2.95 & 6.54  \\
      Dynamic Programming & 4.98 & 3.60 \\
      Graphs & 6.21 & 6.79 \\
      Strings & 23.49 & 8.60 \\
      Geometry & 2.30& 3.77 
    \\\midrule
    \textit{Average} & 9.72 & \textbf{6.35} 
    \\\bottomrule
    \end{tabular}

    \label{tab:clrs_stddev_per_class}
\end{table}

\begin{figure}
    \centering
    \includegraphics[scale=0.6]{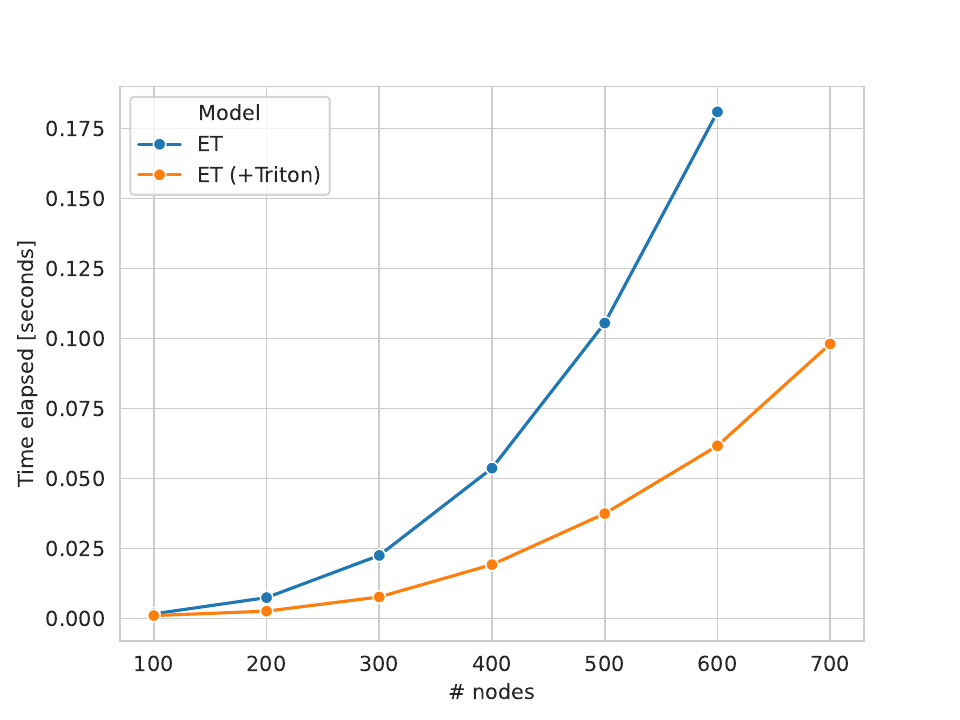}

 \caption{Runtime of the forward pass of a single ET layer in PyTorch in seconds for graphs with up to 700 nodes. We compare the runtime with and without \texttt{torch.compile} (automatic compilation into Triton \citep{Tillet+2019+Triton}) enabled. Without compilation, the ET goes out of memory after 600 nodes.}
    \label{fig:et_timing}
\end{figure}

\begin{table}
\caption{Runtime of a single run of the ET in \textsc{CLRS} on a single A100 GPU.}
\label{tab:avg_runtime_clrs}
    \centering
    \begin{tabular}{lc}
    \toprule
Algorithm & Time in hh:mm:ss \\ \midrule
Activity Selector & 00:09:38 \\
Articulation Points & 01:19:39 \\
Bellman Ford & 00:07:55 \\
BFS & 00:07:03 \\
Binary Search & 00:05:53 \\
Bridges & 01:20:44 \\
Bubble Sort & 01:05:34 \\
DAG Shortest Paths & 00:29:15 \\
DFS & 00:27:47 \\
Dijkstra & 00:09:37 \\
Find Maximum Subarray Kadane & 00:15:25 \\
Floyd Warshall & 00:12:56 \\
Graham Scan & 00:15:55 \\
Heapsort & 00:57:14 \\
Insertion Sort & 00:10:39 \\
Jarvis March & 01:34:40 \\
Kmp Matcher & 00:57:56 \\
LCS Length & 00:08:12 \\
Matrix Chain Order & 00:15:31 \\
Minimum & 00:21:25 \\
MST Kruskal & 01:15:54 \\
MST Prim & 00:09:34 \\
Naive String Matcher & 00:51:05 \\
Optimal BST & 00:12:57 \\
Quickselect & 02:25:03 \\
Quicksort & 00:59:24 \\
Segments Intersect & 00:03:38 \\
Strongly Connected Components & 00:56:58 \\
Task Scheduling & 00:08:50 \\
Topological Sort & 00:27:40 \\
\bottomrule
\end{tabular}
\end{table}

\begin{table}
    \centering
    \caption{Runtime of a single run on the molecular regression datasets, as well as \textsc{BREC}, on L40 GPUs in \textit{days:hours:minutes:seconds}.}
    \begin{tabular}{lcccccc}
    \toprule
    & \textsc{Zinc} (12K) & \textsc{Alchemy} (12K) & \textsc{Zinc-Full} & \textsc{PCQM4Mv2} & \textsc{BREC} \\
    ET & 00:06:04:52 & 00:02:47:51 & 00:23:11:05 & 03:10:35:11 & 00:00:08:37 \\
    ET+RRWP & 00:06:19:52 & 00:02:51:23 & 01:01:10:55 & 03:10:22:06 & - \\\midrule
    Num. GPUs & 1 & 1 & 4 & 4 & 1 \\
    \bottomrule
    \end{tabular}
    \label{tab:total_runtime_mol}
\end{table}

\section{Runtime and memory}\label{app:runtime_and_memory}
Here, we provide additional information on the runtime and memory requirements of the ET in practice. Specifically, in \Cref{fig:et_timing}, we provide runtime scaling of the ET with and without low-level GPU optimizations in PyTorch on an A100 GPU with \texttt{bfloat16} precision. We measure the time for the forward pass of a single layer of the ET on a single graph (batch size of 1) with $n \in \{100, 200, ..., 700\}$ nodes and average the runtime over 100 repeats. We sample random Erd\H{o}s-Renyi graphs with edge probability $0.05$. We use an embedding dimension of 64 and two attention heads.
We find that the automatic compilation into Triton \citep{Tillet+2019+Triton}, performed automatically by \texttt{torch.compile}, improves the runtime and memory scaling. Specifically, with \texttt{torch.compile} enabled, the ET layer can process graphs with up to 700 nodes and shows much more efficient runtime scaling with the number of nodes.

\paragraph{Hardware optimizations}
Efficient compilation of neural networks is already available via programming languages such as Triton \citep{Tillet+2019+Triton}. We use \texttt{torch.compile} in our molecular regression experiments.
In addition, we want to highlight \texttt{FlashAttention} \citep{Dao+2022}, available for the standard transformer, as an example of architecture-specific hardware optimizations that can reduce runtime and memory requirements.

\paragraph{Runtime per dataset/benchmark}
Here, we present additional runtime results for all of our datasets. We present the runtime of a single run on a single L40 GPU of \textsc{Zinc} (12K), \textsc{Alchemy} (12K), and \textsc{BREC}. For \textsc{Zinc-Full} and \textsc{PCQM4Mv2}, we present the runtime of a single run on 4 L40 GPUs; see \Cref{tab:total_runtime_mol}. 

On \textsc{CLRS}, the experiments in our work are run on a mix of A10 and A100 GPUs. To enable a fair comparison, we rerun each algorithm in \textsc{CLRS} in a single run on a single A100 GPU and report the corresponding runtime in \Cref{tab:avg_runtime_clrs}.
Finally, we note that these numbers only reflect the time to run the final experiments and significantly more time was used for preliminary experiments over the course of the research project.

\section{Extended preliminaries}\label{app:notation}
Here, we define our notation. Let $\Nb \coloneqq \{ 1,2,3, \dots \}$. For $n \geq 1$, let $[n] \coloneqq \{ 1, \dotsc, n \} \subset \Nb$. We use $\{\!\!\{ \dots\}\!\!\}$ to denote multisets, i.e., the generalization of sets allowing for multiple instances for each of its elements.

\paragraph{Graphs} A \new{(node-)labeled graph} $G$ is a triple $(V(G),E(G),\ell)$ with \emph{finite} sets of
\new{vertices} or \new{nodes} $V(G)$, \new{edges} $E(G) \subseteq \{ \{u,v\} \subseteq V(G) \mid u \neq v \}$ and a (node-)label function $\ell \colon V(G) \to \Nb$. Then $\ell(v)$ is a \new{label} of $v$, for $v$ in $V(G)$. If not otherwise stated, we set $n \coloneqq |V(G)|$, and the graph is of \new{order} $n$. We also call the graph $G$ an $n$-order graph. 
For ease of notation, we denote the edge $\{u,v\}$ in $E(G)$ by $(u,v)$ or $(v,u)$. We define an $n$-order \new{attributed graph} as a pair $\mathcal{G}=(G,\vec{F})$, where $G = (V(G),E(G))$ and $\vec{F}$ in $\Rb^{n\times p}$ for $p > 0$ is a \new{node feature matrix}. Here, we identify $V(G)$ with $[n]$, then $\vec{F}(v)$ in $\Rb^{1\times p}$ is the \new{feature} or \new{attribute} of the node $v \in V(G)$. Given a labeled graph $(V(G),E(G),\ell)$, a node feature matrix $\vec{F}$ is \new{consistent} with $\ell$ if $\ell(v) = \ell(w)$ for $v,w \in V(G)$ if and only if $\vec{F}(v) = \vec{F}(w)$.

\paragraph{Neighborhood and Isomorphism} The \new{neighborhood} of a vertex $v$ in $V(G)$ is denoted by $N(v) \coloneqq  \{ u \in V(G) \mid (v, u) \in E(G) \}$ and the \new{degree} of a vertex $v$ is  $|N(v)|$. Two graphs $G$ and $H$ are \new{isomorphic} and we write $G \simeq H$ if there exists a bijection $\varphi \colon V(G) \to V(H)$ preserving the adjacency relation, i.e., $(u,v)$ is in $E(G)$ if and only if $(\varphi(u),\varphi(v))$ is in $E(H)$. Then $\varphi$ is an \new{isomorphism} between $G$ and $H$. In the case of labeled graphs, we additionally require that $l(v) = l(\varphi(v))$ for $v$ in $V(G)$, and similarly for attributed graphs. Moreover, we call the equivalence classes induced by $\simeq$ \emph{isomorphism types} and denote the isomorphism type of $G$ by $\tau_G$. %
We further define the atomic type $\text{atp} \colon V(G)^k \to \Nb$, for $k > 0$, such that $\text{atp}(\vec{v}) = \text{atp}(\vec{w})$ for $\vec{v}$ and $\vec{w}$ in $V(G)^k$ if and only if the mapping $\varphi\colon V(G)^k \to V(G)^k$ where $v_i \mapsto w_i$ induces a partial isomorphism, i.e., we have $v_i = v_j \iff w_i = w_j$ and $(v_i,v_j) \in E(G) \iff (\varphi(v_i),\varphi(v_j)) \in E(G)$. 

\paragraph{Matrices} Let $\vec{M} \in \mathbb{R}^{n \times p}$ and $\vec{N} \in \mathbb{R}^{n \times q}$ be two matrices then  $\begin{bmatrix} \vec{M} & \vec{N} \end{bmatrix}  \in \mathbb{R}^{n \times (p+q)}$ denotes column-wise matrix concatenation.  We also write $\Rb^d$ for $\Rb^{1\times d}$. Further, let $\vec{M} \in \mathbb{R}^{p \times n}$ and $\vec{N} \in \mathbb{R}^{q \times n}$ be two matrices then
\begin{equation*}
    \begin{bmatrix} \vec{M} \\ \vec{N} \end{bmatrix} \in \mathbb{R}^{(p+q)\times n}
\end{equation*}
denotes row-wise matrix concatenation. 

For a matrix $\vec{X} \in \mathbb{R}^{n \times d}$, we denote with $\vec{X}_i$ the $i$th row vector. In the case where the rows of $\vec{X}$ correspond to nodes in a graph $G$, we use $\vec{X}_v$ to denote the row vector corresponding to the node $v \in V(G)$.

\paragraph{The Weisfeiler--Leman algorithm}\label{vr_ext} 
We describe the Weisfeiler--Leman algorithm, starting with the \wlone. The \wlone{} or color refinement is a well-studied heuristic for the graph isomorphism problem, originally proposed by~\citet{Wei+1968}.\footnote{Strictly speaking, the \wlone{} and color refinement are two different algorithms. The \wlone{} considers neighbors and non-neighbors to update the coloring, resulting in a slightly higher expressive power when distinguishing vertices in a given graph; see~\cite {Gro+2021} for details. For brevity, we consider both algorithms to be equivalent.} 
Intuitively, the algorithm determines if two graphs are non-isomorphic by iteratively coloring or labeling vertices. Formally, let $G = (V,E,\ell)$ be a labeled graph, in each iteration, $t > 0$, the \wlone{} computes a vertex coloring $C^1_t \colon V(G) \to \Nb$, depending on the coloring of the neighbors. That is, in iteration $t>0$, we set
\begin{equation*}
	C^1_t(v) \coloneqq \REL\Big(\!\big(C^1_{t-1}(v),\oms C^1_{t-1}(u) \mid u \in N(v)  \cms \big)\! \Big),
\end{equation*}
for all vertices $v$ in $V(G)$,
where $\REL$ injectively maps the above pair to a unique natural number, which has not been used in previous iterations. In iteration $0$, the coloring $C^1_{0}\coloneqq \ell$. To test if two graphs $G$ and $H$ are non-isomorphic, we run the above algorithm in ``parallel'' on both graphs. If the two graphs have a different number of vertices colored $c$ in $\Nb$ at some iteration, the \wlone{} \new{distinguishes} the graphs as non-isomorphic. 
It is easy to see that \wlone{} cannot distinguish all non-isomorphic graphs~\citep{Cai+1992}. 

\paragraph{The \texorpdfstring{$k$}{k}-dimensional Weisfeiler--Leman algorithm}\label{kwl_intro} Due to the shortcomings of the $\wlone$ or color refinement in distinguishing non-isomorphic graphs, several researchers, e.g.,~\citet{Bab1979,Cai+1992}, devised a more powerful generalization of the former, today known
as the $k$-dimensional Weisfeiler-Leman algorithm (\kwl{k}), operating on $k$-tuples of nodes rather than single nodes. 

Intuitively, to surpass the limitations of the \wlone, the \kwl{k} colors node-ordered $k$-tuples instead of a single node. More precisely, given a graph $G$, the \kwl{k} colors the tuples from $V(G)^k$ for $k \geq 2$ instead of the nodes. By defining a neighborhood between these tuples, we can define a coloring similar to the \wlone. Formally, let $G$ be a graph, and let $k \geq 2$. In each iteration, $t \geq 0$, the algorithm, similarly to the \wlone, computes a
\new{coloring} $C^k_t \colon V(G)^k \to \Nb$. In the first iteration, $t=0$, the tuples $\vec{v}$ and $\vec{w}$ in $V(G)^k$ get the same
color if they have the same atomic type, i.e.,
$\text{atp}_k(\vec{v})=\text{atp}_k(\vec{u})$. %
Then, for each iteration, $t > 0$, $C^k_{t}$ is defined by
\begin{equation}\label{vr_ext_app}
	C^k_{t}(\vec{v}) \coloneqq \REL \big(C^k_{t-1}(\vec{v}), M_t(\vec{v}) \big),
\end{equation}
with $M_t(\vec{v})$ the multiset
\begin{equation}\label{mi}
	M_t(\vec{v}) \coloneqq  \big( \{\!\! \{  C^{k}_{t-1}(\phi_1(\vec{v},w)) \mid w \in V(G) \} \!\!\}, \dots, \{\!\! \{  C^{k}_{t-1}(\phi_k(\vec{v},w)) \mid w \in V(G) \} \!\!\} \big),
\end{equation}
and where
\begin{equation*}
	\phi_j(\vec{v},w)\coloneqq (v_1, \dots, v_{j-1}, w, v_{j+1}, \dots, v_k).
\end{equation*}
That is, $\phi_j(\vec{v},w)$ replaces the $j$-th component of the tuple $\vec{v}$ with the node $w$. Hence, two tuples are \new{adjacent} or \new{$j$-neighbors} if they are different in the $j$th component (or equal, in the case of self-loops). Hence, two tuples $\vec{v}$ and $\vec{w}$ with the same color in iteration $(t-1)$ get different colors in iteration $t$ if there exists a $j$ in $[k]$ such that the number of $j$-neighbors of $\vec{v}$ and $\vec{w}$, respectively, colored with a certain color is different.

We run the \kwl{k} algorithm until convergence, i.e., until for $t$ in $\Nb$
\begin{equation*}
	C^k_{t}(\vec{v}) = C^k_{t}(\vec{w}) \iff C^k_{t+1}(\vec{v}) = C^k_{t+1}(\vec{w}),
\end{equation*}
for all $\vec{v}$ and $\vec{w}$ in $V(G)^k$ holds. %

Similarly to the \wlone, to test whether two graphs $G$ and $H$ are non-isomorphic, we run the \kwl{k} in ``parallel'' on both graphs. Then, if the two graphs have a different number of nodes colored $c$, for $c$ in $\Nb$, the \kwl{k} \textit{distinguishes} the graphs as non-isomorphic. By increasing $k$, the algorithm gets more powerful in distinguishing non-isomorphic graphs, i.e., for each $k \geq 2$, there are non-isomorphic graphs distinguished by $(k+1)$\text{-}\textsf{WL} but not by \kwl{k}~\citep{Cai+1992}.

\paragraph{The folklore \texorpdfstring{$k$}{k}-dimensional Weisfeiler--Leman algorithm}\label{app:fkwl_intro} A common and well-studied variant of the \kwl{k} is the \fkwl{k}, which differs from the \kwl{k} only in the aggregation function. Instead of~\cref{mi}, the ``folklore'' version of the \kwl{k} updates $k$-tuples according to
\begin{equation*}\label{eq:folklore_wl}
    M^{\text{F}}_t(\vec{v}) \coloneqq  \{\!\! \{ (  C^{k, \text{F}}_{t-1}(\phi_1(\vec{v},w)), ..., C^{k, \text{F}}_{t-1}(\phi_k(\vec{v},w)) ) \mid w \in V(G) \} \!\!\},
\end{equation*}
resulting in the coloring $C^{k,\text{F}}_t \colon V(G)^k \to \Nb$, and is strictly more powerful than the \kwl{k}. Specifically, for $k \geq 2$, the \kwl{k} is exactly as powerful as the \fkwl{(k-1)}~\citep{Gro+2021}.

\paragraph{Computing \kwl{k}'s initial colors} Let $G = (V(G), E(G), \ell)$ be a labeled graph, $k \geq 2$, and let $\vec{v} := (v_1, \dots, v_k) \in V(G)^k$ be a $k$-tuple. Then, we can present the atomic type $\text{atp}(\vec{v})$ by a $k \times k$ matrix $K$ over $\{ 1,2,3\}$. That is, the entry $K_{ij}$ is 1 if $(v_i,v_j) \in E(G)$, 2 if $v_i = v_j$, and 3 otherwise. Further, we ensure consistency with $\ell$, meaning that for two $k$-tuples $\vec{v} := (v_1, \dots, v_k)\in V(G)^k$ and $\vec{w} := (w_1, \dots, w_k)\in V(G)^k$, then
\begin{equation*}
    C^k_0(\mathbf{v}) = C^k_0(\mathbf{w}),
\end{equation*}
if and only if, $\text{atp}(\mathbf{v}) = \text{atp}(\mathbf{w})$ and $\ell(v_i) = \ell(w_i)$, for all $i \in [k]$. Note that we compute the initial colors for both \kwl{k} and the \fkwl{k} in this way.

\subsection{Relationship between first-order logic and Weisfeiler--Leman}
We begin with a short review of \citet{Cai+1992}. 
We consider our usual node-labeled graph $G = (V(G), E(G), \ell)$ with $n$ nodes. However, we replace $\ell$ with a countable set of color relations $C_1, \dots, C_n$, where for a node $v \in V(G)$,
\begin{equation*}
    C_i(v) \Longleftrightarrow \ell(v) = i.
\end{equation*}
Note that \citet{Cai+1992} consider the more general case where nodes can be assigned to multiple colors simultaneously. However, for our work, we assume that a node is assigned to precisely one color, and hence, the set of color relations is at most of size $n$. We can construct first-order logic statements about $G$. For example, the following sentence describes the existence of a triangle formed by two nodes with color $1$:
\begin{equation*}
    \exists x_1 \exists x_2 \exists x_3 \big(E(x_1, x_2) \wedge E(x_1, x_3) \wedge E(x_2, x_3) \wedge C_1(x_1) \wedge C_1(x_2)\big).
\end{equation*}
Here, $x_1$, $x_2$, and $x_3$ are \new{variables} which can be repeated and re-quantified at will.
Statements made about $G$ and a subset of nodes in $V(G)$ are of particular importance to us. To this end, we define a \new{$k$-configuration}, a function $f: \{x_1, \dots, x_k\} \rightarrow V(G)$ that assigns a node in $V(G)$ to each one of the variables $x_1,\ldots,x_k$. Let $\varphi$ be a first-order formula with free variables among $x_1,\dots,x_k$. %
Then, we write
\begin{equation*}
    G, f \models \varphi
\end{equation*}
if $\varphi$ is true when the variable $x_i$ is interpreted as the node $f(x_i)$, for $i=1,\ldots,k$.

\citet{Cai+1992} define the language $\mathcal{C}_{k,m}$ of all first-order formulas with counting quantifiers, at most $k$ variables, and quantifier depth bounded by $m$, and the language $\mathcal{C}_k = \bigcup_{m\geq 0}\mathcal{C}_{k,m}$.
For example, the sentence $\forall x \exists ! 3 y\big( E(x,y) \big)$
in $\mathcal{C}_2$ describes 3-regular graphs; i.e., graphs where each vertex has exactly 3 neighbors.

We define the  equivalence relation $\equiv_{k,m}$ over pairs $(G,f)$ made of graphs $G$ and $k$-configurations $f$ as $(G,f) \equiv_{k,m} (H,g)$ if and only if
\begin{equation*}
G,f \models \varphi \iff H,g \models \varphi    
\end{equation*}
for all formulas $\varphi$ in $\mathcal{C}_{k,m}$ whose free variables are among $x_1,\ldots,x_k$.

We can now formulate a main result of \citet{Cai+1992}. Let $G$ and $H$ be two graphs, let $k\geq 1$ and $m\geq 0$ be non-negative integers, and let $f$ and $g$ be $k$-configurations for $G$ and $H$ respectively.
If $\vec{u}=(f(x_1),\ldots,f(x_k)) \in V(G)^k$ and $\vec{v}=(g(x_1),\ldots,g(x_k)) \in V(H)^k$, then
\begin{equation*}
    C^{k,F}_m(\vec{u}) = C^{k,F}_m(\vec{v}) \iff (G,f) \equiv_{k,m} (H,g) \,.
\end{equation*}

\section{Proofs}\label{app:proof}
Here, we first generalize the GNN from \citet{Gro+2021} to the \fkwl{2}. Higher-order GNNs with the same expressivity have been proposed in prior works by \citet{Azi+2020}. However, our GNNs have a special form that can be computed by the Edge Transformer.

Formally, let $S \subseteq \mathbb{N}$ be a finite subset.
First, we show that multisets over $S$ can be injectively mapped to a value in the closed interval $(0,1)$, a variant of Lemma VIII.5 in~\citet{Gro+2021}. Here, we outline a streamlined version of its proof, highlighting the key intuition behind representing multisets as $m$-ary numbers. 
Let $M \subseteq S$ be a multiset with multiplicities $a_1, \dots, a_k$ and distinct $k$ values. We define the \textit{order} of the multiset as $\sum_{i=1}^k a_i$.
We can write such a multiset as
a sequence $x^{(1)}, \dots, x^{(l)}$ where $l$ is the order of the multiset.
Note that the order of the sequence is arbitrary and that for $i \neq j$ it is possible to have $x^{(i)} = x^{(j)}$. We call such a sequence an $M$-sequence of length $l$.
We now prove a slight variation of a result of~\citet{Gro+2021}.

\begin{lemma}\label{lemma:multiset_m_ary} 
For a finite $m \in \mathbb{N}$, let $M \subseteq S$ be a multiset of order $m - 1$ and let $x_i \in S$ denote the $i$th number in a fixed but arbitrary ordering of $S$.
Given a mapping $g \colon S \rightarrow (0,1)$ where
\begin{equation*}
    g(x_i) \coloneqq m^{-i},
\end{equation*}
and an $M$-sequence of length $l$ given by $x^{(1)}, \dots, x^{(l)}$ with positions $i^{(1)}, \dots, i^{(l)}$ in $S$, the sum
\begin{equation*}
 \sum_{j \in [l]} g(x^{(j)}) = \sum_{j \in [l]} m^{-i^{(j)}}
\end{equation*}
is unique for every unique $M$.
\end{lemma}
\begin{proof}
By assumption, let $M \subseteq S$ denote a multiset of order $m - 1$. Further, let $x^{(1)}, \dots, x^{(l)} \in M$ be an $M$-sequence with $i^{(1)}, \dots, i^{(l)}$ in $S$. Given our fixed ordering of the numbers in $S$ we can equivalently write $M = ( (a_1, x_1), \dots, (a_n, x_n) )$, where $a_i$ denotes the multiplicity of $i$th number in $M$ with position $i$ from our ordering over $S$. Note that for a number $m^{-i}$ there exists a corresponding $m$-ary number written as
\begin{equation*}
    0.0 \ldots \underbrace{1}_{i} \ldots
\end{equation*}
Then the sum,
\begin{align*}
    \sum_{j \in [l]} g(x^{(j)}) &= \sum_{j \in [l]} m^{-i^{(j)}} \\
    &= \sum_{i \in S} a_im^{-i} \in (0,1)
\end{align*}
and in $m$-ary representation
\begin{align*}
   0.a_1 \ldots a_n. 
\end{align*}
Note that $a_i = 0$ if and only if there exists no $j$ such that $i^{(j)} = i$. Since the order of $M$ is $m - 1$, it holds that $a_i < m$. Hence, it follows that the above sum is unique for each unique multiset $M$, implying the result.
\end{proof}

Recall that $S \subseteq \mathbb{N}$ and that we fixed an arbitrary ordering over $S$. Intuitively, we use the finiteness of $S$ to map each number therein to a fixed digit of the numbers in $(0,1)$. The finite $m$ ensures that at each digit, we have sufficient ``bandwidth'' to encode each $a_i$. Now that we have seen how to encode multisets over $S$ as numbers in $(0,1)$, we review some fundamental operations about the $m$-ary numbers defined above. We will refer to decimal numbers $m^{-i}$ as \textit{corresponding} to an $m$-ary number
\begin{equation*}
    0.0 \ldots \underbrace{1}_{i} \ldots,
\end{equation*}
where the $i$th digit after the decimal point is $1$ and all other digits are $0$, and vice versa.

To begin with, addition between decimal numbers implements \textit{counting} in $m$-ary notation, i.e., 
\begin{equation*}
    m^{-i} + m^{-j} \text{ corresponds to } 0.0\ldots \underbrace{1}_{i} \ldots \underbrace{1}_{j} \ldots,
\end{equation*}
for digit positions $i \neq j$ and
\begin{equation*}
    m^{-i} + m^{-j} \text{ corresponds to } 0.0\ldots \underbrace{2}_{i=j}\ldots,
\end{equation*}
otherwise.
We used counting in the previous result's proof to represent a multiset's multiplicities. Next, multiplication between decimal numbers implements \textit{shifting} in $m$-ary notation, i.e.,
\begin{equation*}
    m^{-i} \cdot m^{-j} \text{ corresponds to } 0.0\ldots \underbrace{1}_{i+j}\ldots.
\end{equation*}
Shifting further applies to general decimal numbers in $(0,1)$. Let $x \in (0,1)$ correspond to an $m$-ary number with $l$ digits,
\begin{equation*}
   0.a_1 \ldots a_l.
\end{equation*}
Then,
\begin{equation*}
   m^{-i} \cdot x \text{ corresponds to } 0.0\ldots 0\underbrace{a_1 \ldots a_l}_{i+1, \dots, i+l}.
\end{equation*}

Before we continue, we show a small lemma stating that two non-overlapping sets of $m$-ary numbers preserve their uniqueness under addition.
\begin{lemma}\label{lemma:non_overlapping_m-ary_sets}
Let $A$ and $B$ be two sets of $m$-ary numbers for some $m > 1$. If
\begin{equation*}
    \min_{x \in A} x > \max_{y \in B} y,
\end{equation*}
then for any $x_1, x_2 \in A, y_1, y_2 \in B$,
\begin{equation*}
    x_1 + y_1 = x_2 + y_2 \Longleftrightarrow x_1 = x_2 \text{ and } y_1 = y_2.
\end{equation*}
\end{lemma}
\begin{proof}
The statement follows from the fact that if 
\begin{equation*}
    \min_{x \in A} x > \max_{y \in B} y,
\end{equation*}
then numbers in $A$ and numbers in $B$ do not overlap in terms of their digit range. Specifically, there exists some $l > 0$ such that we can write
\begin{align*}
    x &\coloneqq 0.x_1 \dots x_l \\
    y &\coloneqq 0.\underbrace{0 \dots 0}_l y_1 \dots y_k,
\end{align*}
for some $k > l$ and all $x \in A$, $y \in B$. As a result,
\begin{equation*}
    x + y = 0.x_1 \dots x_l y_1 \dots y_k.
\end{equation*}
Hence, $x + y$ is unique for every unique pair $(x,y)$. This completes the proof.
\end{proof}

We begin by showing the following proposition, showing that the tokenization in \Cref{eq:higher_order_tokens} is sufficient to encode the initial node colors under \fkwl{2}.
\begin{proposition}\label{prop:higher_order_tokens}
Let $G = (V(G), E(G), \ell)$ be a node-labeled graph with $n$ nodes.
Then, there exists a parameterization of~\Cref{eq:higher_order_tokens} with $d = 1$ such that for each $2$-tuples $\vec{u}, \vec{v} \in V(G)^2$,
\begin{equation*}
    C^{2, \text{F}}_0(\vec{u}) = C^{2, \text{F}}_0(\vec{v}) \Longleftrightarrow \vec{X}(\vec{u}) = \vec{X}(\vec{v}).
\end{equation*}
\end{proposition}
\begin{proof}
The statement directly follows from the fact that the initial color of a tuple $\vec{u} \coloneqq (i,j)$ depends on the atomic type and the node labeling. In \Cref{eq:higher_order_tokens}, we encode the atomic type with $\vec{E}_{ij}$ and the node labels with
\begin{equation*}
    \begin{bmatrix}
        \vec{E}_{ij} & \vec{F}_i & \vec{F}_j
    \end{bmatrix}
\end{equation*}
The concatenation of both node labels and atomic type is clearly injective. Finally, since there are at most $n^2$ distinct initial colors of the \fkwl{2}, said colors can be well represented within $\mathbb{R}$, hence there exists an injective $\phi$ in \Cref{eq:higher_order_tokens} with $d=1$. This completes the proof.
\end{proof}

We now show \Cref{theorem:k_fwl}. Specifically, we show the following two propositions from which \Cref{theorem:k_fwl} follows.

\begin{proposition}\label{prop:theorem_forward}
Let $G = (V(G), E(G), \ell)$ be a node-labeled graph with $n$ nodes and $\vec{F} \in \mathbb{R}^{n \times p}$ be a node feature matrix consistent with $\ell$. Then for all $t \geq 0$, there exists a parametrization of the ET such that
\begin{equation*}
    C^{2, \text{F}}_t(\vec{v}) = C^{2, \text{F}}_t(\vec{w}) \Longleftarrow \vec{X}^{(t)}(\vec{v}) = \vec{X}^{(t)}(\vec{w}),
\end{equation*}
for all pairs of $2$-tuples $\vec{v}$ and $\vec{w} \in V(G)^2$.
\end{proposition}
\begin{proof}
We begin by stating that our domain is compact since the ET merely operates on at most $n$ possible node features in $\vec{F}$ and binary edge features in $\vec{E}$, and at each iteration there exist at most $n^2$ distinct \fkwl{2} colors. 
We prove our statement by induction over iteration $t$. For the base case, we can simply invoke \Cref{prop:higher_order_tokens} since our input tokens are constructed according to \Cref{eq:higher_order_tokens}. Nonetheless, we show a possible initialization of the tokenization that is consistent with \Cref{eq:higher_order_tokens} that we will use in the induction step. 

From \Cref{prop:higher_order_tokens}, we know
that the color representation of a tuple can be represented in $\mathbb{R}$. We denote the color representation of a tuple $\vec{u} = (i,j)$ at iteration $t$ as $\vec{T}^{(t)}(\vec{u})$ and $\vec{T}^{(t)}_{ij}$ interchangeably.
We choose a $\phi$ in \Cref{eq:higher_order_tokens} such that for each $\vec{u} = (i,j)$
\begin{equation*}
    \vec{X}^{(0)}_{ij} = \begin{bmatrix}
        \vec{T}^{(0)}_{ij} &
        \Big(\vec{T}^{(0)}_{ij} \Big)^{n^2}
    \end{bmatrix} \in \mathbb{R}^2,
\end{equation*}
where we store the tuple features, one with exponent $1$ and once with exponent $n^2$ and where $\vec{T}^{(0)}_{ij} \in \mathbb{R}$ and $\Big(\vec{T}^{(0)}_{ij} \Big)^{n^2} \in \mathbb{R}$. We choose color representations $\vec{T}^{(0)}_{ij}$ as follows.
First, we define an injective function $f_t: V(G)^2 \rightarrow [n^2]$ that maps each $2$-tuple $\vec{u}$ to a number in $[n^2]$ unique for its \fkwl{2} color $C^{2, \text{F}}_t(\vec{u})$ at iteration $t$. Note that $f_t$ can be injective because there can at most be $[n^2]$ unique numbers under the \fkwl{2}. We will use $f_t$ to map each tuple color under the \fkwl{2} to a unique $n$-ary number.
We then choose $\phi$ in \Cref{eq:higher_order_tokens} such that for each $(i,j) \in V(G)^2$,
\begin{equation*}
    \big|\big| \vec{T}^{(0)}_{ij} - n^{-f_0(i,j)} \big|\big|_F < \epsilon_0,
\end{equation*}
for all $\epsilon_0 > 0$, by the universal function approximation theorem, which we can invoke since our domain is compact. We will use $\Big(\vec{T}^{(0)}_{ij} \Big)^{n^2}$ in the induction step; see below.

For the induction, we assume that
\begin{equation*}
    C^{2,\text{F}}_{t-1}(\vec{v}) = C^{2,\text{F}}_{t-1}(\vec{w}) \Longleftarrow \vec{T}^{(t-1)}(\vec{v}) = \vec{T}^{(t-1)}(\vec{w})
\end{equation*}
and that
\begin{equation*}
    \big|\big| \vec{T}^{(t-1)}_{ij} - n^{-f_{t-1}(i, j)} \big|\big|_F < \epsilon_{t-1},
\end{equation*}
for all $\epsilon_{t-1} > 0$ and $(i,j) \in V(G)^2$.
We then want to show that there exists a parameterization of the $t$-th layer such that
\begin{equation}\label{eq:2fwl_proof_inductive_case}
    C^{2,\text{F}}_t(\vec{v}) = C^{2,\text{F}}_t(\vec{w}) \Longleftarrow \vec{T}^{(t)}(\vec{v}) = \vec{T}^{(t)}(\vec{w})
\end{equation}
and that
\begin{equation*}
    \big|\big| \vec{T}^{(t)}_{ij} - n^{-f_{t}(i, j)} \big|\big|_F < \epsilon_t,
\end{equation*}
for all $\epsilon_t > 0$ and $(i,j) \in V(G)^2$.
Clearly, if this holds for all $t$, then the proof statement follows.
Thereto, we show that the ET updates the tuple representation of tuple $(j,m)$ as
\begin{equation}\label{eq:2fwl_proof_X_update}
    \vec{T}^{(t)}_{jm} = \mathsf{FFN} \Big( \vec{T}^{(t-1)}_{jm} + \frac{\beta}{n} \sum_{l=1}^n \vec{T}^{(t-1)}_{jl} \cdot \Big(\vec{T}^{(t-1)}_{lm}\Big)^{n^2} \Big),
\end{equation}
for an arbitrary but fixed $\beta$. We first show that then, \Cref{eq:2fwl_proof_inductive_case} holds. Afterwards we show that the ET can indeed compute \Cref{eq:2fwl_proof_X_update}.
To show the former, note that for two $2$-tuples $(j,l)$ and $(l, m)$,
\begin{equation*}
   n^{-n^2} \cdot n^{-f_{t-1}(j, l)} \cdot \Big(n^{-f_{t-1}(l, m)}\Big)^{n^2} = n^{-(n^2 + f_{t-1}(j, l) + n^2 \cdot f_{t-1}(l, m))},
\end{equation*}
is unique for the pair of colors
\begin{equation*}
    \big( C^{2,\text{F}}_t((j,l)), C^{2,\text{F}}_t((l,m)) \big)
\end{equation*}
where $n^{-n^2}$ is a constant normalization term we will later introduce with $\frac{\beta}{n}$.
Note further, that we have
\begin{equation*}
    \big|\big| \vec{T}^{(t-1)}_{jl} \cdot \Big(\vec{T}^{(t-1)}_{lm}\Big)^{n^2} - n^{-(n^2 + f_{t-1}(j, l) + n^2 \cdot f_{t-1}(l, m))} \big|\big|_F < \delta_{t-1},
\end{equation*}
for all $\delta_{t-1} > 0$. Further, $n^{-(f_{t-1}(j, l) + n^2 \cdot f_{t-1}(l, m))}$ is still an $m$-ary number with $m = n$.
As a result, we can
set $\beta = n^{-n^2+1}$ and
invoke \Cref{lemma:multiset_m_ary} to obtain that
\begin{equation*}
\frac{\beta}{n} \cdot \sum_{l = 1}^n n^{-(f_{t-1}(j, l) + n^2 \cdot f_{t-1}(l, m))} = \sum_{l = 1}^n n^{-(n^2 + f_{t-1}(j, l) + n^2 \cdot f_{t-1}(l, m))},
\end{equation*}
is unique for the multiset of colors
\begin{equation*}
    \{\!\! \{ ( C^{2, \text{F}}_{t-1}((l,m)), C^{2, \text{F}}_{t-1}((j,l)) ) \mid l \in V(G) \} \!\!\},
\end{equation*}
and we have that
\begin{equation*}
    \big|\big| \frac{\beta}{n} \sum_{l=1}^n \vec{T}^{(t-1)}_{jl} \cdot \Big(\vec{T}^{(t-1)}_{lm}\Big)^{n^2} - \sum_{l = 1}^n n^{-(n^2 + f_{t-1}(j, l) + n^2 \cdot f_{t-1}(l, m))} \big|\big|_F < \gamma_{t-1},
\end{equation*}
for all $\gamma_{t-1} > 0$.
Finally, 
we define
\begin{align*}
    A &\coloneqq \Big\{ n^{-f_{t-1}(j, m)} \mid (j,m) \in V(G)^2 \Big\} \\
    B &\coloneqq \Big\{ \frac{\beta}{n} \cdot \sum_{l = 1}^n n^{-(f_{t-1}(j, l) + n^2 \cdot f_{t-1}(l, m))}  \mid (j,m) \in V(G)^2 \Big\}.
\end{align*}
Further, because we multiply with $\frac{\beta}{n}$, we have that
\begin{equation*}
  \min_{x \in A} x > \max_{y \in B} y 
\end{equation*}
and as a result, by \Cref{lemma:non_overlapping_m-ary_sets},
\begin{equation*}
   n^{-f_{t-1}(j, m)} + \frac{\beta}{n} \cdot \sum_{l = 1}^n n^{-(f_{t-1}(j, l) + n^2 \cdot f_{t-1}(l, m))} %
\end{equation*}
is unique for the pair
\begin{equation*}
    \big(C^{2,\text{F}}_{t-1}((j,m)), \{\!\! \{ ( C^{2, \text{F}}_{t-1}((l,m)), C^{2, \text{F}}_{t-1}((j,l)) ) \mid l \in V(G) \} \!\!\} \big)
\end{equation*}
and consequently for color $C^{2,\text{F}}_t((j,m))$ at iteration $t$. Further, we have that
\begin{equation*}
    \big|\big| \vec{T}^{(t-1)}_{jm} + \frac{\beta}{n} \sum_{l=1}^n \vec{T}^{(t-1)}_{jl} \cdot \Big(\vec{T}^{(t-1)}_{lm}\Big)^{n^2} - n^{-f_{t-1}(j, m)} + \frac{\beta}{n} \cdot \sum_{l = 1}^n n^{-(f_{t-1}(j, l) + n^2 \cdot f_{t-1}(l, m))} \big|\big|_F < \tau_{t-1},
\end{equation*}
for all $\tau_{t-1} > 0$.
Finally, since our domain is compact, we can invoke universal function approximation with \textsf{FFN} 
in \Cref{eq:2fwl_proof_X_update} to obtain
\begin{equation*}
    \big|\big| \vec{T}^{(t)}_{jm} - n^{-f_{t}(j, m)} \big|\big|_F < \epsilon_t,
\end{equation*}
for all $\epsilon_t > 0$. 
Further, because $n^{-f_{t}(j, m)}$ is unique for each unique color $C^{2,\text{F}}_t((j,m))$, \Cref{eq:2fwl_proof_inductive_case} follows.

It remains to show that the ET can indeed compute \Cref{eq:2fwl_proof_X_update}.
To this end, we will require a single transformer head in each layer. 
Specifically, we want this head to compute
\begin{align}
    h_1(\vec{X}^{(t-1)})_{jm} &= \frac{\beta}{n}\sum_{l=1}^n \vec{T}^{(t-1)}_{jl} \cdot \Big(\vec{T}^{(t-1)}_{lm}\Big)^{n^2}  \label{eq:2fwl_proof_head1}.
\end{align}
Now, recall the definition of the Edge Transformer head at tuple $(j,m)$ as
\begin{equation*}
    h_1(\vec{X}^{(t-1)})_{jm} \coloneqq \sum_{l=1}^n  \alpha_{jlm} \vec{V}^{(t-1)}_{jlm},
\end{equation*}
where
\begin{equation*}
    \alpha_{jlm} \coloneqq \underset{l \in [n]}{\mathsf{softmax}} \Big( \frac{1}{\sqrt{d_k}} \vec{X}^{(t-1)}_{jl}\vec{W}^Q  (\vec{X}^{(t-1)}_{lm}\vec{W}^K)^T \Big)
\end{equation*}
with
\begin{equation*}
    \vec{V}_{jlm}^{(t-1)} \coloneqq \vec{X}^{(t-1)}_{jl} \begin{bmatrix}
        \vec{W}^{V_1}_1 \\
        \vec{W}^{V_1}_2
    \end{bmatrix} \odot \vec{X}^{(t-1)}_{lm}\begin{bmatrix}
        \vec{W}^{V_2}_1 \\
        \vec{W}^{V_2}_2
    \end{bmatrix}
\end{equation*}
and by the induction hypothesis above,
\begin{align*}
    \vec{X}^{(t-1)}_{jl} &= \begin{bmatrix}
        \vec{T}^{(t-1)}_{jl} &
        \Big(\vec{T}^{(t-1)}_{jl} \Big)^{n^2}
        \end{bmatrix} \\
    \vec{X}^{(t-1)}_{lm} &= \begin{bmatrix}
        \vec{T}^{(t-1)}_{lm} &
        \Big(\vec{T}^{(t-1)}_{lm} \Big)^{n^2}
        \end{bmatrix},
\end{align*}
where we expanded sub-matrices. Specifically, $\vec{W}^{V_1}_1, \vec{W}^{V_2}_1, \vec{W}^{V_1}_2, \vec{W}^{V_2}_2 \in \mathbb{R}^{ \frac{d}{2} \times d}$.
We then set
\begin{align*}
    \vec{W}^Q &= \vec{W}^K = \vec{0} \\
    \vec{W}^{V_1}_1 &= \begin{bmatrix}
        \beta \vec{I} & \vec{0}
    \end{bmatrix} \\
    \vec{W}^{V_1}_2 &= \begin{bmatrix}
        \vec{0} & \vec{0}
    \end{bmatrix} \\
    \vec{W}^{V_2}_1 &= \begin{bmatrix}
        \vec{0} & \vec{I}
    \end{bmatrix} \\
    \vec{W}^{V_2}_2 &= \begin{bmatrix}
        \vec{0} & \vec{0}
    \end{bmatrix}.
\end{align*}
Here, $\vec{W}^Q$ and $\vec{W}^K$ are set to zero to obtain uniform attention scores.
Note that then for all $j, l, k$, $\alpha_{jlm} = \frac{1}{n}$, due to normalization over $l$, and we end up with \Cref{eq:2fwl_proof_head1} as
\begin{equation*}
   h_1(\vec{X}^{(t-1)})_{jm} = \frac{1}{n} \sum_{l=1}^n  \vec{V}^{(t-1)}_{jlm} %
\end{equation*}
where
\begin{equation*}
    \begin{split}
        \vec{V}_{jlm}^{(t-1)} &= \begin{bmatrix}
        \vec{T}^{(t-1)}_{jl} \cdot \beta \vec{I} +
        \Big(\vec{T}^{(t-1)}_{jl} \Big)^{n^2} \cdot \vec{0} & \vec{0}
        \end{bmatrix} \odot \begin{bmatrix}
        \vec{T}^{(t-1)}_{lm} \cdot \vec{0} +
        \Big(\vec{T}^{(t-1)}_{lm} \Big)^{n^2} \cdot \vec{I}
         & \vec{0}
        \end{bmatrix} \\
        &= \beta \cdot \begin{bmatrix}
        \vec{T}^{(t-1)}_{jl} \cdot \Big(\vec{T}^{(t-1)}_{lm} \Big)^{n^2}
        & \vec{0}
        \end{bmatrix}.
    \end{split}
\end{equation*} 
We now conclude our proof as follows. Recall that the Edge Transformer layer computes the final representation $\vec{X}^{(t)}$ as
\begin{equation*}
\begin{split}
    \vec{X}^{(t)}_{jm} &= \mathsf{FFN} \Bigg( \vec{X}^{(t-1)}_{jm} + h_1(\vec{X}^{(t-1)})_{jm} \vec{W}^O \Bigg) \\
    &=\mathsf{FFN} \Bigg( \begin{bmatrix}
        \vec{T}^{(t-1)}_{jm} & \Big(\vec{T}^{(t-1)}_{jm} \Big)^{n^2}
        \end{bmatrix} + \frac{\beta}{n}\sum_{l=1}^n \begin{bmatrix}
        \vec{T}^{(t-1)}_{jl} \cdot \vec{T}^{(t-1)}_{lm}
        & \vec{0}
        \end{bmatrix} \vec{W}^O \Bigg) \\
    &\underset{\vec{W}^O \coloneqq \vec{I}}{=}\mathsf{FFN} \Bigg( \begin{bmatrix}
        \vec{T}^{(t-1)}_{jm} & \Big(\vec{T}^{(t-1)}_{jm} \Big)^{n^2}
        \end{bmatrix} +  \begin{bmatrix}
        \frac{\beta}{n}\sum_{l=1}^n \vec{T}^{(t-1)}_{jl} \cdot \vec{T}^{(t-1)}_{lm}
        & \vec{0}
        \end{bmatrix} \Bigg) \\
       &=\mathsf{FFN} \Bigg( \begin{bmatrix}
        \vec{T}^{(t-1)}_{jm} + \frac{\beta}{n}\sum_{l=1}^n  \vec{T}^{(t-1)}_{jl} \cdot \vec{T}^{(t-1)}_{lm} & \Big(\vec{T}^{(t-1)}_{jm} \Big)^{n^2}
        \end{bmatrix} \Bigg) \\
        &\underset{Eq. \ref{eq:2fwl_proof_X_update}}{=}\mathsf{FFN} \Bigg( \begin{bmatrix}
        \vec{T}^{(t)}_{jm}  & \Big(\vec{T}^{(t-1)}_{jm} \Big)^{n^2}
        \end{bmatrix} \Bigg)
\end{split}
\end{equation*}
for some $\mathsf{FFN}$. Note that the above derivation only modifies the terms inside the parentheses and is thus independent of the choice of $\mathsf{FFN}$. We have thus shown that the ET can compute \Cref{eq:2fwl_proof_X_update}.

To complete the induction, let $f \colon \mathbb{R}^2 \rightarrow \mathbb{R}^2$ be such that
\begin{equation*}
    f\Bigg(\begin{bmatrix}
        \vec{T}^{(t)}_{jm}  & \Big(\vec{T}^{(t-1)}_{jm} \Big)^{n^2}
        \end{bmatrix}\Bigg) = \begin{bmatrix}
        \vec{T}^{(t)}_{jm}  & \Big(\vec{T}^{(t)}_{jm} \Big)^{n^2}
        \end{bmatrix}.
\end{equation*}
Since our domain is compact, $f$ is continuous, and hence we can choose $\mathsf{FFN}$ to approximate $f$ arbitrarily close. This completes the proof.
\end{proof}

\newcommand{\set}[1]{\ensuremath{\{ #1 \}}}
\newcommand{\multiset}[1]{\ensuremath{\{\!\!\{ #1 \}\!\!\}}}

Next, we show the other direction of \Cref{theorem:k_fwl} under mild and reasonable assumptions. First, we say that a recoloring function, that maps structures over positive integers into positive integers, is \emph{(effectively) invertible} if its inverse is computable. All coloring functions used in practice (e.g., hash-based functions, those based on pairing functions, etc) are invertible. Second, the layer normalization operation is a proper function if it uses statistics collected only during training mode, and not during evaluation mode.

\begin{proposition}\label{prop:theorem_backward}
Let $\REL$ be an invertible function, and let us consider the \fkwl{2} coloring algorithm using $\REL$.
Then, for all parametrizations of the ET with proper layer normalization, for all node-labeled graphs $G=(V(G),E(G),\ell)$, and for all $t\geq 0$:
\begin{equation*}
    C^{2, \text{F}}_t(\vec{v}) = C^{2, \text{F}}_t(\vec{w}) \Longrightarrow \vec{X}^{(t)}(\vec{v}) = \vec{X}^{(t)}(\vec{w}),
\end{equation*}
for all pairs of 2-tuples $\vec{v}$ and $\vec{w}$ in $V(G)^2$.
\end{proposition}
\begin{proof}
We first claim that there is a computable function $Z:\mathbb{N}^*\times\mathbb{N} \rightarrow \mathbb{R}^p$, where $\mathbb{N}^*=\{0\}\cup\mathbb{N}$, such that $\vec{X}^{(t)}(\vec{v})=Z(t,C^{2,F}_t(\vec{v}))$ for all $\vec{v}\in V(G)^2$, independent of the graph $G$ and its order.
The proof of the claim is by induction on $t$. For $t=0$, by definition, $C^{2,F}_0(\vec{v})$ identifies the atomic type $\text{atp}_2(\vec{v})$ which defines $\vec{X}^{(0)}(\vec{v})$ (since the atomic type tells if $v$ is an edge in $G$, and the labels of the vertices in $\vec{v}$).

For $t>0$ and $\vec{v}=(i,j)$, the function $Z(t,C^{2,F}_t(\vec{v}))$ proceeds as follows. First, it uses the invertibility of $\REL$ to obtain the pair
\begin{equation*}
   \bigg( C^{2,F}_{t-1}(i,j), \,\multiset{ \big( C^{2,F}_{t-1}(i,l), C^{2,F}_{t-1}(l,j)\big) \mid l \in V(G) } \bigg)\,.
\end{equation*}
Then, by inductive hypothesis using the function $Z(t-1,\cdot)$, it obtains the pair
\begin{equation*}
    \bigg( \vec{X}^{(t-1)}(i,j), \,\multiset{ \big( \vec{X}^{(t-1)}(i,l), \vec{X}^{(t-1)}(l,j) \big) \mid l \in V(G) } \bigg) \,.
\end{equation*}
Finally, it computes
\begin{equation*}
    \vec{X}^{(t)}(i,j) = \mathsf{FFN}\bigg( \vec{X}^{(t-1)}(i,j) + \triattn\big( \mathsf{LN}\big( \vec{X}^{(t-1)}(i,j) \big) \big) \bigg)
\end{equation*}
under the assumption that the layer normalization is a proper function.
The statement of the proposition then follows directly from the claim since
\begin{equation*}
   \vec{X}^{(t)}(\vec{v}) = Z(t, C^{2,F}_t(\vec{v})) = Z(t, C^{2,F}_t(\vec{w})) = \vec{X}^{(t)}(\vec{w}) \,.
\end{equation*}
\end{proof}

Note that unlike the result in \Cref{prop:theorem_forward}, the above result is uniform, in that the concrete choice of $\REL$ and the function $Z$ does not depend on the graph size $n$.
Finally, \Cref{theorem:k_fwl} follows from \Cref{prop:theorem_forward} and \Cref{prop:theorem_backward}.

\end{document}